\documentclass{article}

\usepackage{microtype}
\usepackage{graphicx}
\usepackage{subfigure}
\usepackage{booktabs} 

\usepackage{hyperref}

\usepackage[accepted]{icml2023}

\usepackage{amsmath}
\usepackage{amssymb}
\usepackage{mathtools}
\usepackage{amsthm}

\usepackage[capitalize,noabbrev]{cleveref}

\theoremstyle{plain}
\newtheorem{theorem}{Theorem}[section]

\newtheorem{lemma}[theorem]{Lemma}

\theoremstyle{definition}

\newtheorem{assumption}[theorem]{Assumption}
\theoremstyle{remark}
\newtheorem{remark}[theorem]{Remark}

\usepackage[textsize=tiny]{todonotes}

\usepackage{enumerate}
\usepackage{pdfsync}
\usepackage[OT1]{fontenc}
\usepackage{smile}
\usepackage{setspace}
\usepackage{tabularx}
\usepackage{float}
\usepackage{wrapfig,lipsum}
\usepackage{enumitem}
\usepackage{pgfplots}
\usetikzlibrary{arrows,shapes,snakes,automata,backgrounds,petri}

\newcommand{\la}{\langle}
\newcommand{\ra}{\rangle}

\allowdisplaybreaks
\usepackage{colortbl}
\definecolor{LightCyan}{rgb}{0.8, 0.9, 1}
\definecolor{LightGray}{gray}{0.9}

\def \algname {\text{algorithm}}

\usepackage{enumitem}
\usepackage{colortbl}
\definecolor{LightCyan}{rgb}{0.8, 0.9, 1}

\usepackage{xcolor}

\ifdefined\final
\usepackage[disable]{todonotes}
\else
\usepackage[textsize=tiny]{todonotes}
\fi
\newcommand{\alglinelabel}{%
  \addtocounter{ALC@line}{-1}
  \refstepcounter{ALC@line}
  \label
}
\icmltitlerunning{Nearly Minimax Optimal Regret for Learning Linear Mixture Stochastic Shortest Path}

\begin{document}

\twocolumn[
\icmltitle{Nearly Minimax Optimal Regret for Learning Linear Mixture \\ Stochastic Shortest Path}



\icmlsetsymbol{equal}{*}

\begin{icmlauthorlist}
\icmlauthor{Qiwei Di}{yyy}
\icmlauthor{Jiafan He}{yyy}
\icmlauthor{Dongruo Zhou}{yyy}
\icmlauthor{Quanquan Gu}{yyy}
\end{icmlauthorlist}

\icmlaffiliation{yyy}{Department of Computer Science, University of California, Los Angeles, CA 90095, USA}

\icmlcorrespondingauthor{Quanquan Gu}{qgu@cs.ucla.edu}

\icmlkeywords{Reinforcement Learning, Markov Decision Process, Stochastic Shortest Path, Linear Function Approximation}

\vskip 0.3in
]



\printAffiliationsAndNotice{}  

\begin{abstract}

We study the Stochastic Shortest Path (SSP) problem with a linear mixture transition kernel, where an agent repeatedly interacts with a stochastic environment and seeks to reach certain goal state while minimizing the cumulative cost. 
Existing works often assume a strictly positive lower bound of the cost function or an upper bound of the expected length for the optimal policy. In this paper, we propose a new algorithm to eliminate these restrictive assumptions. Our algorithm is based on extended value iteration with a fine-grained variance-aware confidence set, where the variance is estimated recursively from high-order moments. Our algorithm achieves an $\tilde{\cO}(dB_*\sqrt{K})$ regret bound, where $d$ is the dimension of the feature mapping in the linear transition kernel, $B_*$ is the upper bound of the total cumulative cost for the optimal policy, and $K$ is the number of episodes. Our regret upper bound matches the $\Omega(dB_*\sqrt{K})$ lower bound of linear mixture SSPs in \citet{min2022learning}, which suggests that our algorithm is nearly minimax optimal.

\end{abstract}    

\def \algname{\text{LEVIS}^{++}}
\section{Introduction}\label{submission}
Stochastic Shortest Path (SSP) \citep{bertsekas2012dynamic} is a type of reinforcement learning problem where the agent aims to reach a predefined goal state while minimizing the total expected cost. In an SSP, for each episode, the agent starts at a specific initial state, chooses an action from the action set, receives some cost from the environment, and transits to the next state. The agent will stop at a fixed goal state (i.e., terminal state) and ends the current episode.  Compared with episodic Markov Decision Processes (MDPs) and infinite-horizon MDPs, the SSP model is more general and thus  more suitable to many modern applications such as Atari games, GO games, and navigation \citep{andrychowicz2017hindsight,nasiriany2019planning}.

For an SSP, since the agent only stops after reaching a goal state, the length of the episode usually depends on the current policy and can be different from episode to episode. Therefore, learning an SSP is usually more difficult than learning episodic MDPs and infinite-horizon MDPs. In recent years, there has been a sequence of works developing efficient algorithms for learning SSPs. We use regret to measure each algorithm, which is defined as the difference between the total cost and the lowest expected cost achieved by the optimal policy. For the tabular SSP setting where the state space and action space are finite, \citet{tarbouriech2020no} proposed an algorithm with a regret of $\tilde{\cO}(D^{3/2}S \sqrt{AK/c_{\text{min}}})$, where $D$ is the smallest expected hitting time from any starting state to the goal state and $c_{\text{min}}$ is the assumed positive lower bound of the cost function. \citet{rosenberg2020near} proposed an algorithm with a regret of $\tilde{\cO}(B_* S\sqrt{AK})$ and showed that every algorithm should suffer from an 
$\Omega(B_*\sqrt{SAK})$ regret. 
Later, \citet{cohen2021minimax} developed an algorithm reduced from algorithms for episodic MDPs, which achieves the minimax lower bound. 
\citet{tarbouriech2021stochastic} made significant contributions to the study of SSP. One of their notable achievements is the development of an algorithm that does not rely on the assumption that $c_{\text{min}} > 0$. This algorithm achieves a nearly optimal regret bound of $\tilde{\cO}(B_*\sqrt{SAK})$. Furthermore, they introduced the algorithms that do not require knowledge of $T_*$ or $B_*$ as well.

Many modern RL problems work with a large state and action spaces. In these cases, linear function approximation can be employed as a tool to make RL scalable to large state and action spaces \citep{bradtke1996linear}. For the SSP setting, \citet{vial2022regret} is the first one to consider linear function approximation on it. They proposed a computationally inefficient algorithm with an $\tilde{\cO}(\sqrt{d^3 B_*^3 K/c_{\text{min}}})$ regret, where $d$ is the dimension of the linear representation used in the algorithm. Later, \citet{chen2022improved} proposed a computationally efficient algorithm with an improved regret of $\tilde{\cO}(\sqrt{d^3 B_*^2 T_* K})$ using the fact that $T_*\leq B_*/c_{\text{min}}$. 
\citet{min2022learning} considered the linear mixture SSP setting and proposed an algorithm $\text{LEVIS}^+$ with a regret of $\tilde{\cO}(d B_*\sqrt{K/c_{\text{min}}})$. \citet{chen2022improved} also proposed an algorithm (UCRL-VTR-SSP) for linear mixture SSP with an $\tilde{\cO}(B_* \sqrt{d T_* K} + dB_*\sqrt{K})$ regret. On the other hand, \citet{min2022learning} proved a lower bound of $\Omega(dB_*\sqrt{K})$. 

However, the regret bounds of all the above works with linear function approximation 
depend on $c_{\text{min}}$ or the expected length $T_*$ polynomially, which prevents these algorithms from matching the lower bound, unlike their counterparts in the tabular setting.  
Therefore, a natural question arises: 
\begin{center}
Can we design an optimal algorithm for linear mixture SSPs, whose regret matches the lower bound? 
\end{center}

Our work gives a positive answer to this question. We highlight our main contributions as follows, 
\begin{itemize}[leftmargin = *]
    \item We propose a computationally-efficient algorithm for learning linear mixture SSPs.
    Our regret bound is $\tilde{\cO}(dB_* \sqrt{K})$, which matches the lower bound in \citet{min2022learning} up to logarithmic factors. To the best of our knowledge, this is the first statistically near-optimal algorithm for learning SSPs with linear function approximation.
    
    \item Our algorithm has a component that estimates the optimal value function by solving a weighted regression problem, following \citep{min2022learning}. The difference between our approach and the previous one is that the weights adapted in our weighted regression depend on both the variance of the estimated value function and the upper bound of the error between the optimal value function and the estimated value function, which has been studied in the design of horizon-free algorithms of linear mixture MDPs \citep{zhou2022computationally}. Our newly adapted weights enable us to obtain a more accurate estimate of the value function and  improve the final regret. 

    \item We also introduce a more delicate variance estimator. To do this, we introduce high-order moment estimates of the value function and build these estimates by solving multiple groups of weighted regressions. Compared with \citet{min2022learning}, our proposed variance estimates are more accurate, which help us eliminate the polynomial dependence of $c_{\text{min}}$ in the regret. 
    
\end{itemize}

\newcolumntype{g}{>{\columncolor{LightCyan}}c}
\begin{table*}[ht]
\caption{Comparison of algorithms of learning SSP in terms of their regret guarantee. }\label{table:11}
\centering
\begin{tabular}{cggg}
\toprule
\rowcolor{white}
Model & Algorithm & Regret
 \\
\midrule
\rowcolor{white}
  &Bernstein-SSP &    \\
\rowcolor{white}
 &\small{\citep{rosenberg2020near}} & \multirow{-2}{*}{$\tilde  O\Big(B_*S\sqrt{AK}\Big)$} \\
 
 \rowcolor{white}
 &ULCVI  &     \\
 \rowcolor{white}
 &\small{\citep{cohen2021minimax}} & \multirow{-2}{*}{$\tilde  O\Big(\sqrt{(B_*^2 + B_*)SAK}\Big)$}\\
  \rowcolor{white}
\multirow{-4}{*}{Tabular SSP}  & EB-SSP  &   \\
 \rowcolor{white}
&\small{\citep{tarbouriech2021stochastic}} & \multirow{-2}{*}{$\tilde  O\Big(\sqrt{(B_*^2+B_*)SAK} + B_* S^2 A\Big)$}\\
\rowcolor{white}

 \rowcolor{white}
  &{Lower Bound } &    \\
   \rowcolor{white}
 &\small{\citep{rosenberg2020near}} &\multirow{-2}{*}{ $\Omega(B_*\sqrt{SAK})$} \\
 
 \midrule

 \rowcolor{white}
&$\text{LEVIS}^+$ &  \\
\rowcolor{white}
 &\small{\citep{min2022learning}} & \multirow{-2}{*}{$\tilde  O\Big(dB_*\sqrt{K/c_{\text{min}}}\Big)$}  \\

 \rowcolor{white}
  &{ UCRL-VTR-SSP } &    \\
   \rowcolor{white}
 &\small{\citep{chen2022improved}} &\multirow{-2}{*}{ $\tilde  O\Big(B_*\sqrt{dT_*K} + dB_*\sqrt{K}\Big)$} \\
 
 &$\algname$ &    \\

 \multirow{-4}{*}{Linear Mixture SSP}&(\textbf{Our work}) &\multirow{-2}{*}{ $\tilde  O\Big(dB_*\sqrt{K} + d^2 B_*$\Big)} \\
 
 \rowcolor{white}
  &{Lower Bound } &    \\
   \rowcolor{white}
 &\small{\citep{min2022learning}} &\multirow{-2}{*}{ $\Omega(dB_*\sqrt{K})$} \\

\bottomrule
\end{tabular}

\end{table*}

\noindent\textbf{Notation.}
For any positive number $n$, we denote by $[n]=\{1,2,3,\ldots,n\}$. We use lowercase letters to denote scalars 
and use lower and uppercase bold face letters to denote vectors
and matrices respectively. For a vector $\xb \in \RR^d$
and matrix $\bSigma \in \RR^{d \times d}$,
we define $\|\xb\|_{\bSigma} = \sqrt{\xb^{\top}\bSigma \xb}$ and define $\|\xb\|_{\infty} = \max_{i} |x_i|$ to be the infinity norm of a vector. For two sequences $\{a_n\}$ and $\{b_n\}$, we write $a_n=O(b_n)$ if there exists an absolute constant $C$ such that $a_n\leq Cb_n$, and we write $a_n=\Omega(b_n)$ if there exists an absolute constant $C$ such that $a_n\geq Cb_n$. We use $\tilde O(\cdot)$ and $\tilde \Omega(\cdot)$ to further hide the logarithmic factors. We use $\ind\{\}$ to denote the indicator function. For $a,b\in \RR$ satisfying $a < b$, we use $[x]_{[a,b]}$ to denote the truncation function $x \cdot \ind\{a \leq x \leq b\} + a \cdot \ind\{x < a\} + b \cdot \ind\{x > b\}$. 
\section{Related Work}
\textbf{Tabular SSP.}
Stochastic Shortest Path (SSP) is a popular variant of Markov Decision Process, which can be dated back to \citet{bertsekas1991analysis,bertsekas2013stochastic,bertsekas2012dynamic}. The regret minimization problem of SSP was first studied by \citet{tarbouriech2020no}, which proposed the first algorithm with a regret of $\tilde{\cO}(D^{3/2}S \sqrt{AK/c_{\text{min}}})$, and a parameter-free algorithm with an $\cO(K^{3/2})$ regret bound. Here $D$ is the smallest expected hitting time from any starting
state to the goal state and $c_{\text{min}}$ is the assumed positive lower bound of the cost function. It was improved by \citet{rosenberg2020near} to $\cO(B_*S\sqrt{AK})$ when $B_*$ is known and $\cO(B_*^{3/2}S\sqrt{AK})$ in the parameter-free case. There is still a $\sqrt{S}$ gap from the lower bound of $\Omega(B_*\sqrt{SAK})$ proved in the same paper. 
 Later, \citet{cohen2021minimax} proposed an algorithm using the technique of reducing SSP to a finite-horizon MDP with a large terminal cost. This algorithm achieves the lower bound, but it requires some prior knowledge of $T_*$, which can be bypassed by using the trivial upper bound  $T_* \leq B_*/ c_{\text{min}}$, and $B_*$ to properly tune the horizon and terminal cost in the reduction. As mentioned in Remark 2 of \citet{cohen2021minimax}, this large dependence on $1/c_{\text{min}}$ will not work well without the assumption $c_{\text{min}} > 0$. 
Concurrently, \citet{tarbouriech2021stochastic} avoided this requirement. 
They first developed an algorithm that knows $T_*$ without assuming $c_{\text{min}} > 0$. This algorithm achieves an $\tilde{\cO}(B_*\sqrt{SAK})$ regret upper bound, matching the lower bound. They also introduced a parameter-free algorithm that does not require knowing $T_*$ in advance. 
For the case where $B_*$ is unknown, \citet{tarbouriech2021stochastic} proposed an algorithm with a ‘doubling trick’ to guess the unknown $B_*$ from scratch. Using the analysis framework called implicit finite horizon approximation, \citet{chen2021implicit} proposed the first model-free algorithm which is minimax optimal under strictly positive costs. They also introduced a model-based minimax optimal algorithm without this assumption that is computationally more efficient. In other aspects of the literature, \citet{jafarnia2021online} introduced the first posterior sampling algorithm for SSP. \citet{tarbouriech2021sample} studied the problem of SSP with access to a generative model. Moreover, \citet{rosenberg2020stochastic, chen2021finding,chen2021minimax} studied the problem with adversarial costs.

\noindent\textbf{RL with Linear Function Approximation.} There exists a large number of works studying RL with linear function approximation \citep{yang2019sample,jin2020provably,du2019good,zanette2020frequentist,wang2020reinforcement,fei2021risk,zhou2021provably,he2021logarithmic,zhou2022computationally}. The counterpart of the SSP we study in episodic MDPs is called linear mixture MDPs, where the transition probability of the MDP is based on a linear mixture model \citep{modi2020sample,jia2020model,ayoub2020model,min2021variance,zhou2021provably}. \citet{zhou2021nearly} proposed an algorithm to achieve a nearly minimax optimal regret bound in episodic MDP. Recently, a new work can achieve horizon-free regret bound for linear mixture MDPs \citep{zhou2022computationally}.
In the SSP setting, \citet{vial2022regret} is the first to study a linear SSP model, which assumes there exist some feature maps and that both the cost function and the transition probability are linear in the feature maps. They proposed a computationally inefficient algorithm with a regret of $\tilde{\cO}(\sqrt{d^3 B_*^3 K/c_{\text{min}}})$. \citet{chen2022improved} improved this result by a computationally efficient algorithm with an $\tilde{\cO}(\sqrt{d^3 B_*^2 T_* K})$ regret. To avoid the undesirable dependency on $T_*$, \citet{chen2022improved} also proposed a computationally inefficient algorithm with a regret bound of $\tilde{\cO}(d^{3.5}B_* \sqrt{K})$ by constructing some confidence sets.\\
\noindent\textbf{Linear Mixture SSP.}
Linear Mixture SSP is a different type of linear function approximation from linear SSP, which was first studied by \citet{min2022learning}. In their work, they proposed an algorithm (LEVIS) with a regret of $\tilde{\cO}(d B_*^{1.5}\sqrt{K/c_{\text{min}}})$ and an improved version ($\text{LEVIS}^+$) with a regret of $\tilde{\cO}(d B_*^{}\sqrt{K/c_{\text{min}}})$. \citet{chen2022improved} proposed another algorithm (UCRL-VTR-SSP) with an $\tilde{\cO}(B_* \sqrt{d T_* K} + dB_*\sqrt{K})$ regret. When $d \geq T_*$, the result is nearly optimal, but in other cases, the dependency on $T_*$ is undesirable. Similar to the tabular setting, this dependency can be bypassed by replacing $T_*$ with the upper bound $T_* \leq B_*/c_{\text{min}}$.

\section{Preliminaries}
\textbf{Stochastic Shortest Path.} 
An SSP is a tuple $\left(\cS, \cA, \PP, c,s_{\text{init}},g \right)$, where:
    $\cS$ is the state space,
    $\cA$ is a finite action space,
    $s_{\text{init}}$ is the initial state and $g \in S$ is the goal state.
    $\PP(s'|s,a)$ is the probability that action $a$ in state $s$  will lead to state $s'$ at the next step, $g$ is an absorbing state, $\PP(g|g,a) = 1$ for all action $a\in \cA$,
    $c$ is a function from $\cS \times \cA$ to $[0,1]$, where $c(s,a)$ is the immediate cost function of taking action $a$ in the state $s$. In addition, we assume that in the goal state $g$, the cost for any action $a\in \cA$ satisfies $c(g,a) = 0 $.
    
\noindent \textbf{Linear Mixture SSP.} We assume that the unknown transition probability $\PP$ is a linear mixture function of feature mapping \citep{modi2020sample, ayoub2020model, zhou2021nearly}. 
\begin{assumption}[Linear Mixture SSP, \citealt{min2022learning}]\label{assumption:linear}
    We assume $\PP(s'|s,a) = \langle\bphi(s'|s,a),\btheta^*\rangle$, with $\|\btheta^*\|_2 \leq 1$, where the feature mapping $\bphi(s'|s,a):\cS\times\cS\times\cA\rightarrow \RR^d$ is known. For simplicity, for any bounded function $V: \cS \rightarrow [0,1]$, we define the notation $\bphi_V$ as following: $\bphi_V(s,a) = \sum_{s^{\prime} \in \cS} \bphi(s^{\prime}|s,a) V(s^{\prime})$ and we also assume $\|\bphi_V(s,a)\|_2 \leq 1$. 
\end{assumption}

\noindent \textbf{Proper Policies.} We consider stationary and deterministic policies in this work, where each of them is a mapping $\pi: \cS \rightarrow \cA$, such that in state $s$, the agent will take action $\pi(s) \in \cA$. A policy $\pi$ is proper if, with probability 1, it can get to the goal state in finite time. This definition of proper policies is the same as that in \citet{tarbouriech2021stochastic,min2022learning}.
We define $\Pi_{\text{proper}}$ to be set of all the proper policies. We make the assumption that $\Pi_{\text{proper}}$ is non-empty.
\begin{assumption}\label{assumption:policy}
At least one proper stationary and deterministic policy exists.
     $\Pi_{\text{proper}} \neq \emptyset$
\end{assumption}
\noindent \textbf{Value Function.} We define the value function and the corresponding Q-function as below.
\begin{align*}
    &V^{\pi}(s) := \lim_{T \rightarrow \infty} \EE\left[\sum_{t=1}^{T}c(s_t,\pi(s_t)) \bigg| s_1 = s \right], \\
    &Q^{\pi}(s,a) := \\
    & \quad \lim_{T \rightarrow \infty} \EE\left[c(s_1,a_1) + \sum_{t=2}^{T}c(s_t,\pi(s_t)) \bigg| s_1 = s, a_1 = a \right].
\end{align*}
For a proper policy $\pi$ and all state-action pair $(s,a)$, we have $V^{\pi}(s), Q^{\pi}(s,a) < \infty$. We define 
\begin{align*}
\PP V(s,a) &= \sum_{s^{\prime} \in S}\PP(s^{\prime}|s,a)V(s^{\prime})\\
&= \sum_{s^{\prime} \in S}\la \bphi(s^{\prime}|s,a), \btheta^*\ra V(s^{\prime})\\
&= \la \bphi_V(s,a),\btheta^* \ra,
\end{align*}
and the Bellman operator as 
\begin{align*}
\cL V(s) = \min_{a \in \cA}\Big\{c(s,a) + \PP V(s,a)\Big\}.
\end{align*}
By satisfying an additional assumption, the lemma presented here shows that we can derive an optimal policy denoted by $\pi^*$, which possesses numerous significant properties.
\begin{lemma}[\citealt{bertsekas1991analysis,yu2013boundedness,tarbouriech2021stochastic}]\label{Lemma:policy} Suppose that Assumption \ref{assumption:policy} holds
and for every improper policy $\pi$, there exists at least one state $s$, such that $V^{\pi}(s) = \infty$, then
there exists an optimal policy $\pi^*$, which is a stationary, deterministic, and proper. What's more, $V^* = V^{\pi^*}$ is the unique solution of the equation $V = \cL V$. 
\end{lemma}
Note that if the cost function is strictly positive, the second assumption inherently satisfied. In the absence of this assumption and with the knowledge of $T_*$, we employ the perturbation technique used in \citet{tarbouriech2021stochastic,min2022learning}
to bypass the second assumption. To clarify the discussion, we first propose an algorithm under the assumption that $c_{\text{min}} > 0$. The regret of this algorithm is proven in Theorem \ref{Theorem:main}. In cases where this assumption does not hold, we introduce a positive parameter $\rho > 0$. We then apply our algorithm to a perturbed problem with a modified cost function defined as $c^\rho_{ k,i} := \rho + c_{k,i}$, where $c_{k,i}$ represents the received cost.
Simultaneously, we adjust the parameter $B_\rho := B + T_*\rho$. This adjustment ensures that our algorithm can handle an upper bound of the modified cost function. We prove the regret bound for this case in Theorem \ref{thm:pertur}.


We denote by $\pi^*$ the optimal policy. We define $V^*(s) = V^{\pi^*}(s) = \min _{\pi \in \Pi^*} V^{\pi}(s)$, $B_* = \max_{s \in \cS}V^*(s)$ and $Q^*(s,a) = Q^{\pi^*}(s,a)$. We assume that we know an upper bound $B \geq B_*$. Denote by $T^\pi(s)$ the expected time that policy $\pi$ takes to reach the goal state $g$ starting from $s$. $T_*$ is defined to be the expected time for the optimal policy to reach goal, i.e. $T_* =\max_{s \in \cS} T^{\pi^*}(s)$.\\
\noindent \textbf{Regret.}  The regret over the total $K$ episodes is defined as 
\begin{align}\label{regret}
    R_K = \sum_{k = 1}^K \sum_{i = 1}^{I_k} c_{k,i} - K V^*(s_{\text{init}}),
\end{align}
where $I_k$ is the length of the $k$-th episode and $c_{k,i} = c(s_{k,i}, a_{k,i})$ is the cost triggered at the $i$-th step in the $k$-th episode.
Our learning goal is to minimize this regret.
\section{The Proposed Algorithm}
\begin{algorithm}[!ht]
\caption{$\algname$}\label{algorithm1}
	\begin{algorithmic}[1]
	\REQUIRE Regularization parameter $\lambda$, confidence radius $\{\hat{\beta}_t\}_{t\geq 1}$, level $L$, variance parameters $\alpha_t$, $\gamma$, $[L] = \{0,1,\cdots, L-1\} $. 
	\STATE Initialize: set $t \leftarrow 1$, $j \leftarrow 0$, $t_0 = 0$,\\
 $\tilde{\bSigma}_{0,l} \leftarrow \lambda \bI$, $\hat{\btheta}_{0,l} \leftarrow 0$, $\tilde{\bbb}_{0,l} \leftarrow 0$,\\ $Q_0(s,\cdot), V_0(s) \leftarrow 1$ for all $ s \neq g$ and 0 otherwise.
        \FOR{$k = 1,2, \ldots, K$}
            \STATE Set $s_t = s_{\text{init}}$.
            \WHILE{$s_t \neq g$}
                \STATE Take action $a_t = \argmin_{a \in \cA} Q_j(s_t,a) $,
                \\receive cost $c_t = c(s_t, a_t)$,
                \\ and next state $s_{t+1} \sim \PP(\cdot|s_t,a_t)$.
                \STATE Set ${\bar{\sigma}_{t,l}}$ $\leftarrow$ Algorithm 3\\
                $\big(\{\bphi_{V_j^{2^l}}(s_t,a_t), \hat{\btheta}_{t,l}, \tilde{\bSigma}_{t,l}, \hat{\bSigma}_{j,l}\}_{l \in [L]}
                , \hat{\beta}_t, \alpha_t, \gamma\big)$.\alglinelabel{algorithm1:line6}
                \STATE For $l\in [L]$, set $\tilde{\bSigma}_{t,l} \leftarrow\tilde{\bSigma}_{t-1,l}$ \\ \qquad $+ \bar{\sigma}_{t,l}^{-2}\bphi_{V_j^{2^l}}(s_t,a_t)\bphi_{V_j^{2^l}}(s_t,a_t)^{\top}$.\alglinelabel{algorithm1:line7}
                \STATE For $l\in [L]$, set $\tilde{\bb}_{t,l} \leftarrow \tilde{\bb}_{t-1,l}$ \\ \qquad $ + \bar{\sigma}_{t,l}^{-2}\bphi_{V_j^{2^l}}(s_t,a_t)V_j^{2^l}(s_{t+1})$.
                \STATE Set $\hat{\btheta}_{t,l} \leftarrow \tilde{\bSigma}_{t,l}^{-1}  \tilde{\bb}_{t,l}$.\alglinelabel{algorithm1:line9}
                \IF{  $\exists l\in [L]$, $\det(\tilde{\bSigma}_{t,l}) \geq 2\det(\tilde{\bSigma}_{t_j,l})$ or $t \geq 2t_j$}
                \alglinelabel{algorithm1:line10}
                    \STATE $j \leftarrow j + 1$.
                    \STATE $t_j \leftarrow t, \epsilon_j \leftarrow \frac{1}{t_j}, q_j \leftarrow \frac{1}{t_j}$.
                    
                    \FOR{$l \in [L]$}
                    \STATE Set $\hat{\bSigma}_{j,l} = \tilde{\bSigma}_{t_j,l}$.
                    \ENDFOR
                    \STATE Set confidence ellipsoid\\
                    \qquad $\hat{\cC}_{j} \leftarrow \Big\{\btheta : \|\hat{\bSigma}_{j,0}^{\frac{1}{2}}(\btheta - \hat{\btheta}_{t_j,0})\|_2 \leq \hat{\beta}_{t_j}\Big\}$.\alglinelabel{alg:line16}
                    \STATE Set ${Q}_j(\cdot, \cdot) \leftarrow \text{DEVI}(\hat{\cC}_{j},\epsilon_j,q_j)$.
                    \STATE Set ${V}_j(\cdot) \leftarrow \min_{a \in \cA} \bar{Q}_j(\cdot,a)$.
                \ENDIF
                \STATE Set $t \leftarrow t + 1$.
            \ENDWHILE
            \STATE Set $j \leftarrow j + 1$.
        \ENDFOR
	
	\end{algorithmic}
\end{algorithm}
In this section, we will 
propose our algorithm for linear mixture SSPs.

\subsection{Algorithm Description}
Our algorithm is displayed in Algorithm \ref{algorithm1}. Generally speaking, Algorithm \ref{algorithm1} follows $\text{LEVIS}^+$ \citep{min2022learning} to construct $\hat\btheta_{t,0}$ as the estimate of the model parameter $\btheta^*$ at the $t$-th step, using a weighted linear regression (Line \ref{algorithm1:line7} to \ref{algorithm1:line9}) with weights $\bar\sigma_{t,0}$ (Line \ref{algorithm1:line6}). In detail, $\hat\btheta_{t,0}$ is the solution of the weighted regression problem 
\begin{align*}
    \hat\btheta_{t,0} &= \argmin_{\btheta \in \RR^d} \bigg\{\lambda \|\btheta\|^2_2\\ 
    & \qquad + \sum_{i = 1}^{t}\left[\la\bphi_{V_j}(s_i,a_i),\btheta\ra - V_j(s_{i+1})\right]^2/\bar{\sigma}_{i,0}^{-2}\bigg\},
\end{align*}
where $j$ is the index
representing the update times of the confidence region. For simplicity of expression, the dependence of the index $j$ on the specific time $i$ within the summation is omitted in the equation. Given $\hat\btheta_{t,0}$, Algorithm \ref{algorithm1} occasionally updates the confidence region $\hat\cC_j$ ($j$ is the index of the update times of the confidence region) in Line \ref{alg:line16} and runs the subalgorithm $\text{LEVIS}^+$ (Algorithm \ref{algorithm2}) to obtain its value function estimates $ Q_j$ and $ V_j$.

Our algorithm, similar to $\text{LEVIS}^{+}$ in \citet{min2022learning}, divides each episode into intervals of different length. The switch between two intervals is triggered by the updating criterion (Algorithm \ref{algorithm1} Line \ref{algorithm1:line10}). In Algorithm \ref{algorithm1}, we define two indices $t$ and $j$, where the number of steps is indexed by $t$ and the number of intervals is represented by $j$. The first updating criterion is based on the determinant of $L$ groups of covariance matrices $\tilde\bSigma_{t,l}, l \in[L]$ of some given features. The definition of the $L$ groups and the features will be revealed afterwards. If at least one of the determinant of covariance matrices is doubled compared with its determinant at the end of the previous step, we will trigger the DEVI process, update the value function, end the current interval and start a new one. (Line \ref{algorithm1:line10}). The first updating criterion cannot guarantee the finite length for each interval, so we follow \citet{min2022learning} and introduce the second updating criterion. If the number of steps $t$ is doubled compared with the index of step $t_j$ at the end of the previous interval, we will end the current interval and start a new one. This criterion can occur at most $\cO(\log T)$ times and will not add too much complexity to the algorithm.

We construct the confidence ellipsoid $\hat{C}_j$ in 
Line \ref{alg:line16} and feed it to Algorithm \ref{algorithm2} to get the estimation of the $Q$-function. We define a constraint set, 
\begin{align}
\cB = \big\{\btheta: \forall (s,a), \la \bphi(\cdot|s,a), \btheta \ra \text{ is a probability \notag
 }\\ \text{distribution and } \la \bphi(s^{\prime}|g,a), \btheta \ra = \ind \{s^{\prime} = g\}\big\}.\label{defbbb}
\end{align}
 It can be shown that $\hat{C}_j \cap \cB$ contains the true parameter $\btheta^*$ with high probability. Then each one-step value iteration in DEVI (Algorithm \ref{algorithm2}, Line \ref{algorithm2:line5}) applies the Bellman operator to the confidence set $\hat{C}_j \cap \cB$, which will find an optimistic estimate to the true optimal value function $V^*$. To prevent DEVI runs infinitely long (since the while loop condition in Line \ref{alg:wwhile} may not be satisfied), we follow \citet{min2022learning} to add the transition bonus $q_j = \ 1 / t_j$ since  the value iteration may not converge without such a bonus \citep{min2022learning}. The additional bias caused by this bonus can be bounded by $\cO(\log T)$ with our choice of $q_j$ through our next following analysis.

\begin{algorithm}[t!]
    \caption{DEVI\citep{min2022learning}}\label{algorithm2}
    \begin{algorithmic}[1]
    \REQUIRE Confidence set $\cC$, error parameter $\epsilon$, set $\cB$ defined in \eqref{defbbb},
    transition bonus $q$.
    \STATE Initialize: $i \leftarrow 0$, $Q^{(0)}(\cdot,\cdot) = 0,$
    \\ $ V^{(0)}(\cdot) = 0$ and $V^{(-1)}(\cdot) = \infty$.
    \STATE Set $Q(\cdot, \cdot) \leftarrow Q^{(0)}(\cdot,\cdot)$.
    \IF{$\cC \cap \cB \neq \emptyset$}
        \WHILE{$\|V^{(i)}-V^{(i+1)}\|_{\infty} \geq \epsilon$}\alglinelabel{alg:wwhile}
            \STATE
            $Q^{(i+1)}(\cdot,\cdot) $  $\leftarrow c(\cdot,\cdot)$
               \\ \qquad $ + (1-q)
            \min_{\btheta \in \cC \cap \cB} \big\la \btheta, \bphi_{V^{(i)}}(\cdot,\cdot) \big\ra$
            \alglinelabel{algorithm2:line5}
            \STATE $V^{(i+1)}(\cdot) \leftarrow \min_{a \in \cA} Q^{(i+1)}(\cdot,a)$
            \STATE Set $i \leftarrow i +1 $.
        \ENDWHILE
        \STATE $\bar{Q}(\cdot,\cdot) \leftarrow Q^{(i+1)}(\cdot,\cdot)$
    \ENDIF
    \STATE Output $\bar{Q}(\cdot,\cdot)$.
    \end{algorithmic}
\end{algorithm}

\begin{algorithm}[ht!]
    \caption{High-order Moment Estimation (HOME)}
    \begin{algorithmic}[1]
    \label{algorithm3}
    \REQUIRE Features $\{\bphi_{t,l}\}_{l \in [L]}$, vector estimators $ \{\hat{\btheta}_{t,l}\}_{l \in [L]} $,
    covariance matrix $\{ \tilde{\bSigma}_{t,l},\hat{\bSigma}_{j,l}\}_{l \in [L]}$,
    \\confidence radius $\hat{\beta}_t$,
    parameters $\alpha_t$, $\gamma$.
    \FOR{$l = 0, 1, \ldots, L-2$}  \STATE Set $[\bar{\VV}_{t,l}V_{j+1}^{2^l}](s_t,a_t) \leftarrow$ \\ \quad $\big[\big\la\bphi_{t,l+1},\hat{\btheta}_{t,l+1}\big\ra\big]_{[0,B^{2^{l+1}}]}
     - \big[\big\la\bphi_{t,l},\hat{\btheta}_{t,l}\big\ra\big]_{[0,B^{2^l}]}^2$.\alglinelabel{algorithm3:line2}
    \STATE Set $E_{t,l} = \min\Big\{1,2\hat{\beta}_{t}\big\|\hat{\bSigma}_{j,l}^{-\frac{1}{2}}\bphi_{t,l}/B^{2^l}\big\|_2\Big\}$\\ \qquad $ + \min\Big\{1,\hat{\beta}_{t}\big\|\hat{\bSigma}_{j,l+1}^{-\frac{1}{2}}\bphi_{t,l+1}/B^{2^{l+1}}\big\|_2\Big\}$.
    \STATE Set $\sigma_{t,l}^2 \rightarrow [\bar{\VV}_{t,l}V_{j+1}^{2^l}](s_t,a_t)/B^{2^{l+1}} + E_{t,l}$.
    \STATE Set $\bar{\sigma}_{t,l}^2 \leftarrow B^{2^{l+1}}\max \Big\{ \sigma_{t,l}^2, \alpha_t^2,$\\ \qquad $ \gamma^2 \big\|\tilde{\bSigma}_{t,l}^{-\frac{1}{2}}\bphi_{t,l}/B^{2^{l}}\big\|_2\Big\}$.
    \ENDFOR
    \STATE Set $\bar{\sigma}_{t,L-1}^2 \leftarrow B^{2^{L}}\max \Big\{ 1, \alpha_t^2,$\\ \qquad $ \gamma^2 \big\|\tilde{\bSigma}_{t,L-1}^{-\frac{1}{2}}\bphi_{t,L-1}/B^{2^{L-1}}\big\|_2\Big\}$
    \ENSURE $\{\bar{\sigma}_{t,l}^2\}_{l \in [L]}$.
    \end{algorithmic}
\end{algorithm}

\subsection{Comparison with $\text{LEVIS}^+$ \citep{min2022learning}}
In this subsection, we will compare our algorithm with $\text{LEVIS}^+$ \citep{min2022learning}. 
Before telling the key difference between Algorithm \ref{algorithm1} and $\text{LEVIS}^+$, we first recall the proof of the algorithm $\text{LEVIS}^+$ in \citet{min2022learning} to see several key technical challenges we need to recover. 


$\text{LEVIS}^+$ applies weighted ridge regression to obtain their estimate to the $\theta^*$ by the regression weights $\hat{\sigma}_t$. First we rewrite the regret definition $R_K$ by another formulation of double summation, which is 
\begin{align}
    R_K + K V^*(s_{\text{init}})= \sum_{k = 1}^K \sum_{i = 1}^{I_k} c_{k,i}  = \sum_{m = 1}^M \sum_{h = 1}^{H_m} c_{m,h},
\end{align}
where $m$ is a regrouping index of $k$, where the $m$-th interval has the end points either from the end of an episode or the time steps when the confidence region is updated, $H_m$ is the length of $m$-th interval. The following theorem plays a central role in \citet{min2022learning}'s proof.

\begin{theorem}[Theorem G.2 in \citealt{min2022learning}]\label{Theorem:min}
     For any $\delta>0$, let $\rho=0, \lambda=1 / B^2$. Supposing that $c_{\text{min}} > 0$, $K \geq d^5 + B^2d^4/c_{\text{min}}$, then with probability at least $1-7 \delta$, The regret of algorithm $\text{LEVIS}^+$ satisfies
\begin{align*}
R_K= \widetilde{\mathcal{O}}\left(\sqrt{B^2 d T+B^2 d^2 M
}\right),
\end{align*}
where $\widetilde{\mathcal{O}}(\cdot)$ hides a term of $C \cdot \log ^2\left(T B /\left(\lambda \delta c_{\min }\right)\right)$ for some problem-independent constant $C$, and $c_{\min}$ is a minimum of the cost function.

\end{theorem}

However, Theorem \ref{Theorem:min} cannot directly provide a $O(\sqrt{K})$ upper bound for the regret since in the SSP setting, the total number of steps $T$ can be much greater than episode $K$. In order to control the number of steps $T$, \citet{min2022learning} proved the following upper bound of the total length $T$, 
\begin{align*}
    T = \cO\bigg(\frac{KB}{c_{\text{min}}}\log^2\bigg(\frac{KB}{\lambda \delta c_{\text{min}}}\bigg)\bigg),
\end{align*}
which brings a $c_{\text{min}}$ dependency to the regret. Therefore, if we want to achieve our goal to remove the polynomial dependency of $c_{\text{min}}$ from our regret, one way is to remove the $T$ dependency from the regret in Theorem \ref{Theorem:min}. 

In the proof of Theorem \ref{Theorem:min}, \citet{min2022learning} decomposed the regret in the following way: with high probability, we have the following decomposition of the regret,
    \begin{align*}
        R_K &\leq 
     \underbrace{\sum_{m=1}^M \sum_{h=1}^{H_m}\big[V_{j_m}(s_{m,h+1}) - \PP V_{j_m}(s_{m,h},a_{m,h})\big]}_{I_1}\\
     & + \underbrace{\sum_{m=1}^M \sum_{h=1}^{H_m}\big[c_{m,h} + \PP V_{j_m}(s_{m,h},a_{m,h}) - V_{j_m}(s_{m,h})\big]}_{I_2}\\
    &+ \sum_{m=1}^M \sum_{h=1}^{H_m} \big[ V_{j_m}(s_{m,h})-V_{j_m}(s_{m,h+1}) \big] \\
    &  - \sum_{m \in \cM(m)} V_{j_m}(s_{\text{init}}) + 1.
    \end{align*}
Here the first and second terms are dominant, denoted by $I_1$ and $I_2$. The bounds of $I_1$ and $I_2$ in \citet{min2022learning} bring a $\sqrt{T}$ dependency in the results. We will go through the process of their proof and see why the claims hold. 

For term $I_1$, they derived the following result: with probability at least $1-\delta$, the following inequality holds
    \begin{align*}
        &\sum_{m=1}^{M}\sum_{h=1}^{H_m} [V_{j_m}(s_{m,h+1}) - \PP V_{j_m}(s_{m,h},a_{m,h})]
          \\
         &  \qquad \leq 2 B_* \sqrt{2T\log \bigg(\frac{2T}{\delta}\bigg)}.
    \end{align*}
    
This inequality will result in the $\sqrt{T}$ dependency in Theorem \ref{Theorem:min}. Thus, we need to make a more delicate estimation for term $I_2$.

For term $I_2$, \citet{min2022learning} proved the following result.
With high probability, the following inequality holds
\begin{align*}
    &\sum_{m \in \cM_0(M)} \sum_{h=1}^{H_m}\big[c_{m,h} + \PP V_{j_m}(s_{m,h},a_{m,h}) - V_{j_m}(s_{m,h})\big]\\
    &\leq \underbrace{\sum_{m \in \cM_0(M)}\sum_{h=1}^{H_m}\min \Big\{B_*,4\hat{\beta}_{T}\big\|\bphi_{V_{j_m}}(s_{m,h},a_{m,h})\big\|_{{\bSigma}_{t}^{-1}}\Big\}}_{B_1}\\
    & \qquad + \underbrace{(B_* + 1)\sum_{m \in \cM_0(M)}\sum_{h=1}^{H_m}\frac{1}{t_{j_m}}}_{B_2},
\end{align*}
where term $B_2$ is non-dominant by the following inequality from \citet{min2022learning}
\begin{align*}
    B_2 \leq 4.5 B_* \bigg[\log\bigg( 1+ \frac{TB_*^2d}{\lambda}\bigg) + \log T\bigg].
\end{align*}
For term $B_1$, \citet{min2022learning} proved that  
\begin{align*}
    &B_1 \leq \sqrt{\sum_{m \in \cM_0(M)}\sum_{h=1}^{H_m}\Big(B_* + 4\hat{\beta}_T\hat{\sigma}_{t}\Big)^2} \\
    &\cdot \sqrt{\sum_{m \in \cM_0(M)}\sum_{h=1}^{H_m}\min \Big\{1,\|\bphi_{V_{j_m}}(s_{m,h},a_{m,h}/\hat{\sigma}_t\|_{\bSigma_t^{-1}}^2\Big\}}\\
    & \leq \sqrt{2d\log(1 + T/\lambda) \cdot \Big(2TB_*^2 + 32 \hat{\beta}_T^2 \sum_{m \in \cM_0(M)}\sum_{h=1}^{H_m} \hat{\sigma}_t^2}\Big).
\end{align*}
Firstly, this result has a $\sqrt{T}$ dependency. Furthermore, the term $\sum_{m \in \cM_0(M)}\sum_{h=1}^{H_m} \hat{\sigma}_t^2 = \cO(T)$ will also provide a $\sqrt{T}$ dependency. Therefore, if we want to improve the dependency of $c_{\text{min}}$, we need to use more delicate bound for $B_1$. Furthermore, we will also need a new construction of weights $\hat\sigma_t$ with a smaller upper bound.


\subsection{Our Key Techniques}

In this subsection, we will highlight the key techniques used in our algorithm, which tackle the technical challenges faced by $\text{LEVIS}^+$.

\noindent\textbf{Design of Variance-aware and Uncertainty-aware Weights.}
From the above analysis, we can see that a 'good' selection of the weights $\bar\sigma_{t,0}$ can potentially help us to improve the dependency of $T$ in the final regret. The first notable difference between our $\bar\sigma_{t,0}$ and $\hat\sigma_{t}$ in \citet{min2022learning} is that our weights are both \emph{variance-aware} and \emph{uncertainty-aware}. More specifically, 
\begin{align*}
\bar{\sigma}_{t,0}^2 &\leftarrow B^2\max \Big\{ \big[\bar{\VV}_{t,0}V_{j+1}](s_t,a_t)/B^{2} + E_{t,0}, \\& \qquad  \alpha_t^2, \gamma^2 \big\|\tilde{\bSigma}_{t,0}^{-\frac{1}{2}}\bphi_{t,0}/B\big\|_2\Big\},
\end{align*}
where $\bar{\VV}_{t,0}$ is the estimated variance operator we will introduce in the next section, $E_{t,0}$ is low-order error correction term. To compare with, $\hat\sigma_{t}$ only depends on the variance term. \citet{zhou2022computationally} used the uncertainty-aware term to avoid the dependency of the worst-case range of the noise in a regression problem (Theorem 4.1 in \citet{zhou2022computationally}). For our setting, the noise represents the uncertainty of the value function $V_j$, and its range has a worst-case upper bound which depends on $c_{\text{min}}$. Thus, similar to \citet{zhou2022computationally}, our regret improves the $c_{\text{min}}$ polynomially.


\noindent\textbf{High-order Moment Variance Estimator.}
\label{section:HOME}
Besides the change of the definition we have mentioned 
, our adapted weights are more refined by using a recursive construction from the high-order moments estimates of the value function. A similar technique was used by \citet{zhou2022computationally} to achieve a horizon-free result for linear mixture MDPs. Compared with $\text{LEVIS}^+$, from Algorithm \ref{algorithm1} Line \ref{algorithm1:line6} to Algorithm \ref{algorithm1} Line \ref{algorithm1:line9}, we maintain $L$ groups of weights $\bar\sigma_{t,l}$, $l \in [L]$ instead of 2. The $l$-th group includes the $2^l$ moment of the value function for each $l \in [L]$ and do the weighted regression according to each group to obtain different  $\btheta_{t,l}$. A central part of our algorithm is the recursive design of the variance. Intuitively, the variance of a function $V$ is defined as $\VV V = \EE[V^2] - \EE[V]^2$.
Since $\hat{\btheta}_{t,1}$ is from the regression of $V_j^2$ and $\hat{\btheta}_{t,0}$ is from the regression of $V_j$, it's natural to define the variance estimator 
\begin{align*}
\bar{\VV}_{t,0} V_j= [\la \bphi_{V_j^2}, \hat{\btheta}_{t,1} \ra]_{[0,B^2]} - [\la\bphi_{V_j},\hat{\btheta}_{t,0}\ra]^2_{[0,B]}.
\end{align*}
The truncation is used since the optimal value function should satisfy $V^* \leq B$ by our assumption. Similarly, when we try to estimate the variance of some high-order terms of $V_j$, we need to use the regression for some higher-order terms. We design the estimated variance of the high-order terms in Algorithm \ref{algorithm3} Line \ref{algorithm3:line2}, which is
\begin{align*}
&[\bar{\VV}_{t,l}V_{j+1}^{2^l}](s_t,a_t)\\ \notag 
& = \big[\big\la\bphi_{t,l+1},\hat{\btheta}_{t,l+1}\big\ra\big]_{[0,B^{2^{l+1}}]}
 - \big[\big\la\bphi_{t,l},\hat{\btheta}_{t,l}\big\ra\big]_{[0,B^{2^l}]}^2.
\end{align*}
In this way, the information of high-order terms is transferred to the estimate of the lower-order terms, which makes our first-level estimate $\hat{\btheta}_{t_j,0}$ affected by all the high-order information and different from its counterpart in \citet{min2022learning}. The details are shown in Algorithm \ref{algorithm3}.


To be consistent with our higher-order moment estimation technique, we modify the updating criterion, which will be triggered when any of the determinants of the covariance matrix is doubled or the time is doubled, shown in Algorithm \ref{algorithm1} Line \ref{algorithm1:line10}. In detail, the updating criterion will be triggered if there exists any $l \in [L]$, determinant $\tilde\bSigma_{t,l}$ is doubled. This is different from \citet{min2022learning}, which only considers $\tilde\bSigma_{t,0}$.

\section{Main Results}

We show the regret guarantee of Algorithm \ref{algorithm1} as follows. Assume $c_{\text{min}}$ is known in prior, then we have the following guarantees:
\begin{theorem}[Known $c_{\text{min}}$]\label{Theorem:main}
    Set $\alpha_t = 1/\sqrt{t}$, $\gamma = d^{-1/4}$, $\lambda = 1/B^2$, $L=\log (5B/c_{\text{min}}) / \log 2$. With the assumption of $c_{\text{min}} > 0$, for any $\delta>0$, set $\left\{\widehat{\beta}_t\right\}_{t \geq 1}$ as
\begin{align}
\hat{\beta}_t=& 12 \sqrt{ d \log \left(1+t^2 /(d \lambda) \right) \log \left(128(\log (t/d)+2) t^4 / \delta\right)} \notag\\
&+ 30\sqrt{d}\log \left(128(\log (t/d )+2) t^4 / \delta\right) + 1.\label{eq:beta}
\end{align}
then with probability at least $1-(2 L+1) \delta$, the regret of Algorithm \ref{algorithm1} is bounded by
\begin{align*}
R_K \leq \tilde{\cO}(d^2 B + dB \sqrt{K}).
\end{align*}
Here $\tilde \cO$ hides some logarithmic factors of $K$, $B$, $1/c_{\text{min}}$ and $1/ \delta$.
\end{theorem}

\begin{remark}
    If we know $B_*$, then we can set $B = B_*$ and get the regret bound of $\text{Regret}(K) \leq \tilde{\cO}(d^2B_* + dB_*\sqrt{K})$, which is nearly optimal when $K \geq d^2$. That nearly matches the $\cO(dB_*\sqrt{K})$ lower bound proposed by \citet{min2022learning}. 
\end{remark}

With the knowledge of $T^*$, we can build a variant of Algorithm \ref{algorithm1} which does not need know $c_{\text{min}}$ in prior, by using the perturbing technique introduced in \citet{tarbouriech2021stochastic}.  Specifically, let $\rho>0$ be some positive parameter. We run Algorithm \ref{algorithm1} with the cost defined as $c^{\rho}_{k,i}: = \rho + c_{k,i}$, where $c_{k,i}$ is the received cost. Meanwhile, we set $B^\rho := B + T_*\rho$. We call the algorithm as $\rho$-LEVIS$^{++}$. It is easy to see that  $c^{\rho}_{k,i}$ enjoys a uniform lower bound $\rho$ instead of $c_{\text{min}}$. Thus, based on the regret over $c^\rho_{k,i}$ and $B^\rho$ for $\rho$-LEVIS$^{++}$ derived from Theorem \ref{Theorem:main}, we have the following regret bound over $c_{k,i}$ and $B$:

\begin{theorem}[Known $T_*$]\label{thm:pertur}
    Set $\rho = (T_*K)^{-1}$, $\alpha_t = 1/\sqrt{t}$, $\gamma = d^{-1/4}$, $\lambda = 1/B^2$, $L=\log (5B/\rho) / \log 2$. For any $\delta>0$, set $\left\{\widehat{\beta}_t\right\}_{t \geq 1}$ the same as \eqref{eq:beta},
then with probability at least $1-(2 L+1) \delta$, the regret of $\rho$-LEVIS$^{++}$ is bounded by
\begin{align*}
R_K \leq \tilde{\cO}(d^2 B + dB \sqrt{K}).
\end{align*}
Here $\tilde \cO$ hides some logarithmic factors of $K$, $B$, and $1/ \delta$.
\end{theorem}

\section{Proof Sketch}

In this section, we will provide the proof sketch of our main results. 

\subsection{Analysis of DEVI}\label{section:DEVI}
To analyze DEVI, \citet{min2022learning} proved an important fact that the DEVI process will converge and that the true parameter $\btheta^*$ will lie in the confidence ellipsoid we construct. Using this fact, we can get a bound of the error between the estimated variance and the true variance. Another important fact is that with high probability, the output of value function $V_j \leq V^*$, thus $V_j \leq B$ (optimism). But this result is not enough for our purpose because we add the variance of the high-order moment of $V_j$. So we need to bound the error between the estimation of the variance and the true variance of the higher moments. To deal with this problem, we first use the same argument in \citet{min2022learning} to prove the optimism property. For the variance estimation part, we will do induction on $l$ and $j$. It's similar to the technique used in \citet{zhou2022computationally}.
Our result is summarized in the following lemma.  
\begin{lemma}\label{LEMMA:VARIANCE ESTIMATOR}
    Set $\{\hat{\beta}_t\}$ as \eqref{eq:beta}, then with probability at least $1 - L\delta$, for all $t$ and $j = j(t) \geq 1$ as the index of the value functions $V$ at step $t$, and $l \in [L]$, DEVI converges in finite time and the following holds
    \begin{align*}
        & \qquad 0 \leq Q_j(\cdot,\cdot) \leq Q^{*}(\cdot,\cdot), \qquad \btheta^{*} \in \hat{\cC}_{j,l},\notag \\
        &  \Big|\big[\bar{\VV}_{t,l}V_{j}^{2^l}](s_t,a_t)-\big[\VV V_{j}^{2^l}\big](s_t,a_t)\Big| \leq B^{2^{l+1}}E_{t,l}.
    \end{align*}
    
\end{lemma}
\subsection{Regret Decomposition}
Following \citet{min2022learning}, we divide the time steps into intervals. We will divide a new interval every time the DEVI condition (Line \ref{algorithm1:line10}) is triggered. 
Denoted by $M$, the total number of intervals, we decompose the regret based on intervals. This lemma follows \citet{min2022learning}.
\begin{lemma}
\label{LEMMA:REGRET DECOMPOSITION}
Under the event of Lemma \ref{LEMMA:VARIANCE ESTIMATOR}, the regret defined in \eqref{regret} can be decomposed as
    \begin{align*}
        R_K &\leq 
        \underbrace{\sum_{m=1}^M \sum_{h=1}^{H_m}\big[c_{m,h} + \PP V_{j_m}(s_{m,h},a_{m,h}) - V_{j_m}(s_{m,h})\big]}_{I_1} \\
    & \qquad + \underbrace{\sum_{m=1}^M \sum_{h=1}^{H_m}\big[V_{j_m}(s_{m,h+1}) - \PP V_{j_m}(s_{m,h},a_{m,h})\big]}_{I_2}\\
    &\qquad + \sum_{m=1}^M \sum_{h=1}^{H_m} \big[ V_{j_m}(s_{m,h})-V_{j_m}(s_{m,h+1}) \big]\\
    & \qquad \underbrace{\qquad - \sum_{m \in \cM(m)} V_{j_m}(s_{\text{init}}) + 1. \qquad \qquad}_{I_3}\\
    \end{align*}
\end{lemma}
\textbf{Bounding $I_1$ and $I_2$.} Roughly speaking, $I_1$ is the accumulated Bellman error of the DEVI outputs. Recall from Line \ref{algorithm2:line5} in Algorithm \ref{algorithm2} that we add a transition bonus $q_j$ for the sake of convergence, which will cause some bias from the Bellman operator. We choose the same transition bonus $q_j = 1/t_j$ as \citet{min2022learning}. The result is shown in the following lemma,
\begin{lemma}\label{lemma:E1}
Under the event of Lemma \ref{LEMMA:VARIANCE ESTIMATOR}, we have the following bound for $I_1$
\begin{align*}
    &I_1-2 \notag \\
    &\leq \sum_{m \in \cM_0(M)}\sum_{h=1}^{H_m}\big[c_{m,h} + \PP V_{j_m}(s_{m,h},a_{m,h}) - V_{j_m}(s_{m,h})\big]\\
    &\leq \sum_{m \in \cM_0(M)}\sum_{h=1}^{H_m}\min \Big\{B_*,4\hat{\beta}_{T}\big\|\bphi_{V_{j_m}}(s_{m,h},a_{m,h})\big\|_{\tilde{\bSigma}_{t,0}^{-1}}\Big\}\\
    & \quad + (B_* + 1)\sum_{m \in \cM_0(M)}\sum_{h=1}^{H_m}\frac{1}{t_{j_m}}.
\end{align*}
\end{lemma}
Here the term $B_2$ is the bias brought by the transition bonus $q_j$.
\citet{min2022learning} proved that the bound of $B_2$ term can be reduced to bounding the total number of calls to DEVI.
For the rest of term $I_1$ (especially term $B_1$), and term $I_2$, we combine them together as the first term of a sequence so that we can use some recursive analysis techniques, which are similar to that in \citet{lattimore2012pac} and \citet{zhou2022computationally}. 

We define three sequences of quantities, which are
\begin{small}
\begin{align*}
    R_l &= \sum_{m \in \cM_0(M)}\sum_{h=1}^{H_m}\min \Big\{1,\hat{\beta}_{T}\big\|\bphi_{V_{j_m}^{2^l}}(s_{m,h},a_{m,h})/B^{2^{l}}\big\|_{\tilde{\bSigma}_{t,l}^{-1}}\Big\},\\
    S_l &= \sum_{m \in \cM_0(M)}\sum_{h=1}^{H_m}\VV V_{j_m}^{2^l}(s_{m,h},a_{m,h})/ B^{2^{l+1}},\\
    A_l &= \sum_{m \in \cM_0(M)}\sum_{h=1}^{H_m}\Big[\PP V_{j_m}^{2^l}(s_{m,h},a_{m,h})
    -V_{j_m}^{2^l}(s_{m,h+1})\Big]/B^{2^l}.
\end{align*}
\end{small}
Observe that $B_1 \leq 4BR_0$ and $I_2 = B A_0$. Therefore, it suffices to bound $A_0 + 4R_0$. Note that if we have the optimism property ($V_j \leq V^* \leq B$), which is proved to occur with high probability in Section \ref{section:DEVI}, these quantities will be less than $T$, $\forall l$.  
We can find an relationship of $A_l + 4R_l$ and $A_{l+1} + 4R_{l+1}$ by using the Bernstein concentration inequalities. Then, we can get the bound of $A_0 + 4R_0$; thus the bound of $B_1 + I_2$. The detailed proof is in the appendix.

\textbf{Bounding $I_3$.}
\citet{min2022learning} shows that the bound of $E_3$ is reduced to bounding the number of DEVI calls. Since in our algorithm, DEVI will also be triggered when the determinant of the covariance matrix for high-order terms is doubled, the number of DEVI calls is larger than that in \citet{min2022learning}. Fortunately, we prove that the error between the two numbers of DEVI calls only differs a logarithmic term, which does not hurt the final regret heavily. It is shown as the following lemma. 
\begin{lemma}
Conditioned on the event of Lemma \ref{LEMMA:VARIANCE ESTIMATOR}, choose parameter $\alpha_t = {1}/{\sqrt{t}}$, $L=\log (5B/c_{\text{min}}) / \log 2$. The total number of calls to DEVI $J$ is bounded by $J \leq 4 dL \log \left(1+{T}/{\lambda}\right)+2 \log T$. 
\end{lemma}

\section{Conclusions}
We study the problem of Shortest Stochastic Path, where the transition probability is approximated by a linear mixture model. We propose a novel algorithm and prove its regret upper bound. Our result 
nearly matches the lower bound. The hyperparameters of our algorithm still depends on  $c_{\text{min}}$ or $T_*$, and we leave the development of a parameter-free algorithm as that in \citet{tarbouriech2021stochastic} to future work.
\section*{Acknowledgements}
We thank the anonymous reviewers for their helpful comments. QD, JH, DZ and QG are supported in part by the  National Science Foundation CAREER Award 1906169, the Sloan Research Fellowship and the JP Morgan Faculty Research Award. The views and conclusions contained in this paper are those of the authors and should not be interpreted as representing any funding agencies.
\bibliography{reference}
\bibliographystyle{icml2023}
\newpage
\appendix
\onecolumn

\section{Notations}
\begin{table}[ht]
\centering
\begin{tabular}{cl}
\hline
Notation & Meaning \\
\hline
$t$,$T$  & The number of steps./ The total number of steps.\\
\rowcolor{LightGray}
$j$ & The index of the update times of the confidence region.\\
$s_t,a_t$ & States and actions our algorithm encounters at time $t$.\\
\rowcolor{LightGray}
$c_t$ & The cost obtained with state $s_t$ and action $a_t$\\
$l$,$L$& The level of value function used in the regression./ The maximum level.\\
\rowcolor{LightGray}
$Q_j(\cdot,\cdot),V_j(\cdot)$ & \makecell[l]{The Q-function and value function obtained 
in the $j$-th update of the confidence region.} \\
$\pi^*$ & The optimal policy.\\
\rowcolor{LightGray}
$Q^*(\cdot,\cdot),V^*(\cdot)$ & \Gape[0pt][2pt]{\makecell[l]{The Q-function and value function obtained 
with the optimal policy.\\ }} \\
$\hat\btheta_{t,l}$ & \Gape[0pt][2pt]{\makecell[l]{The estimated parameter of the linear mixture SSP model obtained \\by weighted regression at step $t$ with level $l$.}} \\
\rowcolor{LightGray}
$\btheta^*$ & The ground-truth parameter of linear mixture SSP. \\
$\hat{\cC}_j$ & \makecell[l]{The confidence set obtained in the $j$-th update of the confidence region, 
\\containing $\btheta^*$ with high probability.} \\
\rowcolor{LightGray}
$\hat{\beta}_t$ & Confidence radius at step $t$. \\
$\tilde\Sigma_{t,l}$ & The covariance matrix of step $t$ and level $l$.\\
\rowcolor{LightGray}
$\hat\Sigma_{j,l}$ & The covariance matrix used to construct $\hat{\cC}_j$, equal to $\tilde\Sigma_{t_j,l}$.\\
$\bar\sigma_{t,l}$ & \Gape[0pt][2pt]{\makecell[l]{The weights for regression problems \\of step $t$ and level $l$ , defined in Algorithm \ref{algorithm3}.}} \\
\rowcolor{LightGray}
$\alpha_t$, $\gamma$& Adjustable hyperparameters in the definition of $\bar\sigma_{t,l}$.\\
\hline
\end{tabular}   
\caption{Important Notations}
\label{tab:notations}
\end{table}
\section{Numerical Simulations}
We follow the experiment setup in \citet{min2022learning} and perform an experiment to compare our algorithm (LEVIS++) and the algorithm LEVIS  in \citet{min2022learning}. The details are presented as follows.

The action space $\cA = \{-1,1\}^{d-1}$ with $|\cA| = 2^{d-1}$. The state space is $(s_{\text{init}},g)$. We choose $\delta, \Delta$ and $B_*$ such that $\delta + \Delta = 1/B_*$ and $\delta > \Delta$. The true
model parameter $\btheta^*$ is given by 
\begin{align*}
\btheta^* = \left[ \frac{\Delta}{d-1},\ldots,\frac{\Delta}{d-1},1 \right]^{\top}.
\end{align*}
The feature mapping is defined as 
\begin{align*}
\bphi(s_{\text{init}},|s_{\text{init}}, \ab) &= [-\ab, 1 - \delta]^{\top},\\
\bphi(s_{\text{init}}|g, \ab) & = \textbf{0},\\
\bphi(g|s_{\text{init}}, \ab) & = [\ab, \delta]^{\top},\\
\bphi(g|g, a) & = [\textbf{0}_{d-1}, 1]^{\top}.
\end{align*}
This is a linear mixture
SSP with transition function:
\begin{align*}
\PP(s_\text{init}|s_\text{init}, \ab) &= 1 - \delta - \la \ab, \btheta \ra,\\
\PP(g|s_\text{init}, \ab) &= \delta + \la \ab, \btheta \ra,\\
\PP(g|g, a) &= 1,\\
\PP(s_\text{init}|g, a) &= 0.
\end{align*}
\begin{figure}[H] 
    \centering 
    \includegraphics[width=0.7\textwidth]{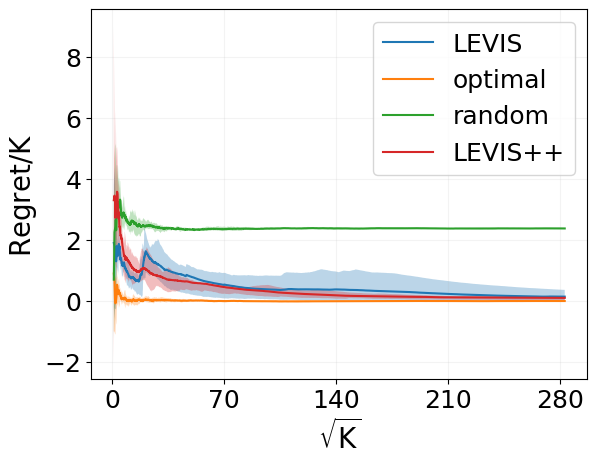} 
    \caption{The plot shows the average regret (i.e. $R_K/K$) and compares the implementation results of Algorithm \ref{algorithm1} and LEVIS in \citet{min2022learning} on the SSP instance described in Appendix B with $\lambda = 1, \rho = 0$ and failing probability 0.01.} 
    \label{graph} 
\end{figure}

As demonstrated in the graph, our algorithm, LEVIS++, achieves consistently smaller regret than LEVIS in \citet{min2022learning}. It is worth noting that the expected length for the optimal policy on this synthetic data is relatively small. (There are only two states and the probability of the optimal policy reaching the goal state from the initial state is $1/B_*$, where $B_*$ is set to be 3. Thus the expected total length for every episode is approximately 3.) One of our primary contributions is the elimination of the polynomial dependency on the expected length $T_*$ for the optimal policy. As a result, our algorithm is likely to be more advantageous in numerous real-world applications where the expected length $T_*$ for the optimal policy is even larger.

\section{Analysis of Algorithm}
\subsection{Analysis of DEVI}
We define some confidence ellipsoids of different levels. For each $j\in \NN$,
let $\hat{\cC}_{j,l} = \Big\{ \btheta: \big\| \hat{\bSigma}_{t_j,l}^{\frac{1}{2}}(\hat{\btheta}_{t_j,l}-\btheta)\big\|_2 \leq \hat{\beta}_{t_j}\Big\}$, $l \in [L]$. $\hat{C}_j = \cap_{l \in [L]}\hat{C}_{j,l}$.
The next lemma shows that with high probability, $\btheta^{*}$ lies in the confidence sets we construct.
\begin{lemma}(Restatement of Lemma \ref{LEMMA:VARIANCE ESTIMATOR})
    Set $\{\hat{\beta}_t\}$ as \eqref{eq:beta}, then with probability at least $1 - L\delta$, for all $t$ and $j = j(t) \geq 1$ as the index of the value functions $V$ at step $t$, and $l \in [L]$, DEVI converges in finite time and the following holds
    \begin{align*}
        \quad 0 \leq Q_j(\cdot,\cdot) \leq Q^{*}(\cdot,\cdot), \quad  \btheta^{*} \in \hat{\cC}_{j,l}, \quad \text{and}\quad  \Big|\big[\bar{\VV}_{t,l}V_{j}^{2^l}](s_t,a_t)-\big[\VV V_{j}^{2^l}\big](s_t,a_t)\Big| \leq B^{2^{l+1}}E_{t,l}.
    \end{align*} 
\end{lemma}
\subsection{Regret Decomposition}
We prove the decomposition of regret following the structure in \citet{min2022learning}. First, we define some notations. The interval is indexed by $1,2,3,\ldots$ and the total number of intervals is denoted by $M$. For the $m$-th interval, the length is denoted by $H_m$. We denote by $\cM(M)$ the set of intervals which are the first interval of
their corresponding episodes. The regret is decomposed as the following lemma shows,
\begin{lemma}(Restatement of Lemma \ref{LEMMA:REGRET DECOMPOSITION})
Under the event of Lemma \ref{LEMMA:VARIANCE ESTIMATOR}, for the regret defined in \eqref{regret}, we have the following decomposition:
    \begin{align*}
        R(M) &\leq 
        \underbrace{\sum_{m=1}^M \sum_{h=1}^{H_m}\big[c_{m,h} + \PP V_{j_m}(s_{m,h},a_{m,h}) - V_{j_m}(s_{m,h})\big]}_{I_1} \\
    & \qquad + \underbrace{\sum_{m=1}^M \sum_{h=1}^{H_m}\big[V_{j_m}(s_{m,h+1}) - \PP V_{j_m}(s_{m,h},a_{m,h})\big]}_{I_2}\\
    &\qquad + \underbrace{\sum_{m=1}^M \sum_{h=1}^{H_m} \big[ V_{j_m}(s_{m,h})-V_{j_m}(s_{m,h+1}) \big] - \sum_{m \in \cM(M)} V_{j_m}(s_{\text{init}}) + 1}_{I_3}.\\
    \end{align*}
\end{lemma}
\section{Proof of Theorem \ref{Theorem:main} and Theorem \ref{thm:pertur}}
Denote by $\cM_0(M)$ the set of $m$ such that $j_m \ge 2$. 

We first deal with the term $I_1$. Without too much confusion, term $I_1$ can be seen as the accumulated Bellman error of the DEVI outputs. We first divide $I_1$ into two terms, the main term caused by the estimation error of vector $\btheta^*$
and the second term caused by the transition bonus $q=1/t_{j}$, which is shown in the following lemma:
\begin{lemma}\label{LEMMA:E1}
Under the event of Lemma \ref{LEMMA:VARIANCE ESTIMATOR}, we have the following inequality,
\begin{align*}
    \sum_{m \in \cM_0(M)} &\sum_{h=1}^{H_m}\big|c_{m,h} + \PP V_{j_m}(s_{m,h},a_{m,h}) - V_{j_m}(s_{m,h})\big|\\
    &\leq \sum_{m \in \cM_0(M)}\sum_{h=1}^{H_m}\min \Big\{B_*,4\hat{\beta}_{T}\big\|\bphi_{V_{j_m}}(s_{m,h},a_{m,h})\big\|_{\tilde{\bSigma}_{t,0}^{-1}}\Big\} + (B_* + 1)\sum_{m \in \cM_0(M)}\sum_{h=1}^{H_m}\frac{1}{t_{j_m}}.
\end{align*}
Furthermore, we have
\begin{align*}
|I_1| \leq 2 + \sum_{m \in \cM_0(M)}\sum_{h=1}^{H_m}\min \Big\{B_*,4\hat{\beta}_{T}\big\|\bphi_{V_{j_m}}(s_{m,h},a_{m,h})\big\|_{\tilde{\bSigma}_{t,0}^{-1}}\Big\} + (B_* + 1)\sum_{m \in \cM_0(M)}\sum_{h=1}^{H_m}\frac{1}{t_{j_m}}.
\end{align*}
\end{lemma}
For simplicity, we define the two terms in the lemma \ref{LEMMA:E1} as $B_1$ and $B_2$, where we have
\begin{align}
    &\sum_{m \in \cM_0(M)} \sum_{h=1}^{H_m}\big[c_{m,h} + \PP V_{j_m}(s_{m,h},a_{m,h}) - V_{j_m}(s_{m,h})\big] \notag\\
    &\leq \underbrace{\sum_{m \in \cM_0(M)}\sum_{h=1}^{H_m}\min \Big\{B_*,4\hat{\beta}_{T}\big\|\bphi_{V_{j_m}}(s_{m,h},a_{m,h})\big\|_{\tilde{\bSigma}_{t,0}^{-1}}\Big\}}_{B_1} + \underbrace{\sum_{m \in \cM_0(M)}\sum_{h=1}^{H_m}\frac{B_*+1}{t_{j_m}}}_{B_2},\label{eq:0001}
\end{align}
and we will bound them separately,
To bound the $B_2$ term, we have the following lemmas.
\begin{lemma}
\label{lemma:J}
Conditioned on the event in Lemma \ref{LEMMA:VARIANCE ESTIMATOR}, if we choose parameter $\alpha_t$ to be $\alpha_t = {1}/{\sqrt{t}}$, then the total number of calls to DEVI algorithm is bounded by $J \leq 4 dL \log \left(1+{T}/{\lambda}\right)+2 \log T$. Furthermore, we have $|M_0| \leq J \leq 4 dL \log \left(1+{T}/{\lambda}\right)+2 \log T$.
\end{lemma}
\begin{lemma}\label{lemma:B2}
Using the same condition in Lemma \ref{lemma:J} and the definition of $t_{j_m}$ in algorithm \ref{algorithm1}, we have an upper bound of $B_2$ term as follows,
\begin{align}
    B_2 = \sum_{m \in \cM_0(M)}\sum_{h=1}^{H_m}\frac{B_*+1}{t_{j_m}} \leq 5(B_*+1) \bigg[2 dL \log (1+T / \lambda)+ \log T\bigg].\notag
\end{align}
\end{lemma}
\begin{proof}[Proof of Lemma \ref{lemma:B2}]
    By rewriting the summation using the index $j$, using Lemma \ref{lemma:J}, we have
\begin{align*}
B_2 &\leq\left(B_{*}+1\right) \sum_{j=0}^J \sum_{t=t_j+1}^{t_{j+1}} \frac{1}{t_j} \leq (B_* + 1)(J+1)\\
& \leq 5(B_*+1) \bigg[2 dL \log (1+T / \lambda)+ \log T\bigg].
\end{align*}
\end{proof}
We then bound the term $I_3$, which follows the proof of Lemma D.4 in \citet{min2022learning}.
\begin{lemma}\label{lemma:E3}
    Assuming the event in Lemma \ref{LEMMA:VARIANCE ESTIMATOR} holds, then we have the following inequality: 
    \begin{align*}
         \sum_{m=1}^M \sum_{h=1}^{H_m} \big[ V_{j_m}(s_{m,h})-V_{j_m}(s_{m,h+1}) \big] - \sum_{m \in \cM(m)} V_{j_m}(s_{\text{init}}) + 1 \leq 2 + 4dB_{*}L \log{\Big(1+\frac{T}{\lambda}\Big)} + 2 B_{*} \log{T}.
    \end{align*}
\end{lemma}
For each level $l \in [L]$, we define the following sequences:
\begin{align}
    R_l &= \sum_{m \in \cM_0(M)}\sum_{h=1}^{H_m}\min \Big\{1,\hat{\beta}_{T}\big\|\bphi_{V_{j_m}^{2^l}}(s_{m,h},a_{m,h})/B^{2^{l}}\big\|_{\tilde{\bSigma}_{t,l}^{-1}}\Big\}.\\\label{eq:SS}
    S_l &= \sum_{m \in \cM_0(M)}\sum_{h=1}^{H_m}\VV V_{j_m}^{2^l}(s_{m,h},a_{m,h})/ B^{2^{l+1}}.\\\label{eq:AA}
    A_l &= \sum_{m \in \cM_0(M)}\sum_{h=1}^{H_m}\big(\PP V_{j_m}^{2^l}(s_{m,h},a_{m,h})-V_{j_m}^{2^l}(s_{m,h+1})\big)/B^{2^l}.
\end{align}
Our goal is to construct the relationship between these sequences, and the following lemma deals with the sequence $R_l$.
\begin{lemma}\label{lemma:R}
For the sequence $R_l$, the following inequality holds,
\begin{align*}
    R_l \leq 2d\iota + 2\hat{\beta}_T \gamma^2 d\iota + 2\sqrt{d\iota}\hat{\beta}_T
    \sqrt{\sum_{m \in \cM_0(M)}\sum_{h=1}^{H_m}\alpha_{t(m,h)}^2 + \sum_{m \in \cM_0(M)}\sum_{h=1}^{H_m}\sigma_{t,l}^2},
\end{align*}
where $\iota = \log (1 + T/d\lambda\alpha_T^2)$.
Next, we calculate the term of the sum of the variance in Lemma \ref{lemma:R}.
\begin{lemma}\label{Lemma:sigma}
Using the variance $\sigma_{t,l}$ and confidence bonus $E_{t,l}$ defined in Algorithm \ref{algorithm3}, and under the event of Lemma \ref{LEMMA:VARIANCE ESTIMATOR}, the following inequality holds, 
\begin{align*}
    \sum_{m \in \cM_0}\sum_{h=1}^{H_m}\sigma_{t,l}^2
    &\leq 2\sum_{m \in \cM_0}\sum_{h=1}^{H_m}E_{t,l} + \sum_{m \in \cM_0}\sum_{h=1}^{H_m}\VV V_{j_m}^{2^l}(s_{m,h},a_{m,h})/B^{2^{l+1}}.
\end{align*}
\end{lemma}
Note that
\begin{align}
E_{t,l} &= \min\Big\{1,2\hat{\beta}_{t}\big\|\hat{\bSigma}_{j,l}^{-\frac{1}{2}}\bphi_{t,l}/B^{2^l}\big\|_2\Big\} + \min\Big\{1,\hat{\beta}_{t}\big\|\hat{\bSigma}_{j,l+1}^{-\frac{1}{2}}\bphi_{t,l+1}/B^{2^{l+1}}\big\|_2\Big\}\notag\\
& \leq \min\Big\{1,\sqrt{2}\hat{\beta}_{t}\big\|\tilde{\bSigma}_{t,l}^{-\frac{1}{2}}\bphi_{t,l}/B^{2^l}\big\|_2\Big\} + \min\Big\{1,\frac{\sqrt{2}}{2}\hat{\beta}_{t}\big\|\tilde{\bSigma}_{t,l+1}^{-\frac{1}{2}}\bphi_{t,l+1}/B^{2^{l+1}}\big\|_2\Big\}\notag,
\end{align}
where the last inequality holds due to Lemma \ref{lemma:det} and the fact $\det (\tilde{\bSigma}_{t,l}) \leq 2\det (\hat{\bSigma}_{j,l})$.

Thus, according to the definition of $R_l$, we get the following inequality,
\begin{align}
    \sum_{m \in \cM_0}\sum_{h=1}^{H_m}E_{t,l} & \leq 2R_l + R_{l+1},\label{eq:E}
\end{align}
Combining the results in Lemma \ref{lemma:R}, Lemma \ref{Lemma:sigma}, and \eqref{eq:E}
, we have
\begin{align}
R_l \leq 2d\iota + 2\hat{\beta}_T \gamma^2 d\iota + 2\sqrt{d\iota}\hat{\beta}_T
    \sqrt{\sum_{m \in \cM_0(M)}\sum_{h=1}^{H_m}\alpha_{t(m,h)}^2 +S_l + 4R_l + 2R_{l+1}}.\label{eq:RR}
\end{align}
In the argument above, we have connected the sequence $\{R_l\}$ with the sequence $\{S_l\}$, and the next two lemmas can connect the bound of $S_l$ with the bound of $A_l$.
\begin{lemma}\label{Lemma:S}
If we define $C_M$ to be the sum of the cost of all steps, i.e. $C_M = \sum_{m=1}^{M}\sum_{h=1}^{H_m} c(s_t,a_t)$. Then we have
\begin{align*}
    S_l \leq A_{l+1} + |\cM_0| + \frac{2^{l+1}}{B}C_M + \frac{2^{l+1}}{B}(4 BR_0 + B_2),
\end{align*}    
where $B_2$ is the bias from the transition bonus and is defined in \eqref{eq:0001}.
\end{lemma}
\begin{lemma}\label{lemma:Bernstein} Let $\left\{S_l, A_l\right\}_{l \in [L]}$ be defined in \eqref{eq:SS} and \eqref{eq:AA}. Then we have $\mathbb{P}\left(\mathcal{E}_{\ref{lemma:Bernstein}}\right)>1-L\delta$, where
\begin{align*}
\mathcal{E}_{\ref{lemma:Bernstein}}:=\left\{\forall l \in [L],\left|A_l\right| \leq \sqrt{2 \zeta S_l}+\zeta\right\} .
\end{align*}
and $\zeta = 4\log{(2\log (T\log (1/\delta))+1)/\delta)}$.
\end{lemma}
\begin{proof}[Proof of Theorem \ref{Theorem:main}]
Under the high-probability event $\mathcal{E}_{\ref{lemma:Bernstein}}$, combining \eqref{eq:RR} with the result in Lemma \ref{Lemma:S}
, we have the following inequalities for the sequence $A_l$ and $R_l$,
\begin{align*}
|A_l| &\leq 2\sqrt{2\Big[|A_{l+1}| + |\cM_0| + \frac{2^{l+1}}{B} C_M + \frac{2^{l+1}}{B}(4BR_0+B_2)\Big]\zeta }  + \zeta,\\
    R_l & \leq 2d\iota + 2\hat{\beta}_T \gamma^2 d\iota\\
    & \qquad + 2\sqrt{d\iota}\hat{\beta}_T
     \sqrt{A_{l+1} + \frac{2^{l+1}}{B}(4BR_0+B_2+C_M) + 4R_l + 2R_{l+1}}\\
    &  \qquad + 2\sqrt{d\iota}\hat{\beta}_T
    \sqrt{\sum_{m \in \cM_0(M)}\sum_{h=1}^{H_m}\alpha_{t(m,h)}^2 + |\cM_0|},
\end{align*}
where we use the fact that $\sqrt{a+b} \leq \sqrt{a} + \sqrt{b}$ when $a, b \geq 0$.

Our goal is to bound $|A_0| + 4R_0$. Using the fact that $\sqrt{a} + \sqrt{b} \leq \sqrt{2(a+b)}$, we have
\begin{align*}
    |A_l| + 4R_l &\leq 8d\iota + 8\hat{\beta}_T \gamma^2 d\iota + \zeta\\
    & \qquad + 8\sqrt{d\iota}\hat{\beta}_T
     \sqrt{A_{l+1} + \frac{2^{l+1}}{B}(4BR_0+B_2+C_M) + 4R_l + 2R_{l+1}}\\
    &  \qquad + 8\sqrt{d\iota}\hat{\beta}_T
    \sqrt{\sum_{m \in \cM_0(M)}\sum_{h=1}^{H_m}\alpha_{t(m,h)}^2 + |\cM_0|}\\
    & \qquad + 2\sqrt{2\bigg[|A_{l+1}| + |\cM_0| + \frac{2^{l+1}}{B} C_M + \frac{2^{l+1}}{B}(4BR_0+B_2)\bigg]\zeta }\\
    & \leq 8d\iota + 8\hat{\beta}_T \gamma^2 d\iota + \zeta
    + 8\sqrt{d\iota}\hat{\beta}_T
    \sqrt{\sum_{m \in \cM_0(M)}\sum_{h=1}^{H_m}\alpha_{t(m,h)}^2 + |\cM_0|}\\
    & \qquad + 2\max\{8 \hat{\beta}_T\sqrt{d\iota},2\sqrt{2\zeta}\}
    \sqrt{|A_{l}| + 4 R_{l} + |A_{l+1}| + 4 R_{l+1}+ 2^{l+1}\Big(|A_0| + 4R_0+\frac{B_2+C_M}{B}\Big) }.
\end{align*}
Set $a_l = A_l + 4R_l$ to Lemma \ref{lemma:zhangnew}. Noting that $a_l = A_l + 4R_l \leq 5T$ and $|A_0| + 4R_0+\frac{B_2+C_M}{B} \geq C_M / B \geq c_{\text{min}} T / B$, our choice of $L = \log{(5B/c_{\text{min}})}/ \log 2$ can satisfy the condition of Lemma \ref{lemma:zhangnew}. Thus, we can get an upper bound for $A_0 + 4R_0$, which is
\begin{align*}
    A_0 + 4R_0 &\leq 22 \bigg( 2\max\Big\{8 \hat{\beta}_T\sqrt{d\iota},2\sqrt{2\zeta}\Big\}\bigg)^2 \\
    & \qquad + 6\Big(8d\iota + 8\hat{\beta}_T \gamma^2 d\iota + \zeta
    + 8\sqrt{d\iota}\hat{\beta}_T
    \sqrt{\sum_{m \in \cM_0(M)}\sum_{h=1}^{H_m}\alpha_{t(m,h)}^2 + |\cM_0|}\Big)\\
    & \qquad + 4\sqrt{2}\Big( 2\max\{8 \hat{\beta}_T\sqrt{d\iota},2\sqrt{2\zeta}\}\Big)\sqrt{|A_0| + 4R_0+\frac{B_2+C_M}{B}}.
\end{align*}

Using the fact that $x \leq a\sqrt{x} + b \Rightarrow x \leq 2a^2 + 2b$, we can further bound $A_0 + 4R_0$ with 
\begin{align}
    A_0 + 4R_0 \notag&\leq 216\Big(\max\{8 \hat{\beta}_T\sqrt{d\iota},2\sqrt{2\zeta}\}\Big)^2\\
    \notag& \qquad + 12\Big(8d\iota + 8\hat{\beta}_T \gamma^2 d\iota + \zeta
    + 8\sqrt{d\iota}\hat{\beta}_T
    \sqrt{\sum_{m \in \cM_0(M)}\sum_{h=1}^{H_m}\alpha_{t(m,h)}^2 + |\cM_0|}\Big)\\
    \label{eq:A0}& \qquad + 8\sqrt{2}\Big( 2\max\{8 \hat{\beta}_T\sqrt{d\iota},2\sqrt{2\zeta}\}\Big)\sqrt{\frac{B_2+C_M}{B}}.
\end{align}
Finally, by the decomposition of regret, we the regret is upper bounded  by $R(M) \leq I_1 + I_2 + I_3$, where $I_1 \leq 2 + 4BR_0 + B_2$ due to Lemma \ref{LEMMA:E1} and the definition of $R_0$, $I_2 = BA_0$ by the definition of $A_0$, $I_3 \leq 2 + 4dB_{*}L \log{\Big(1+{T}/{\lambda}\Big)} + 2 B_{*} \log{T}$ by Lemma \ref{lemma:E3}. Combining all of these results, we have 
\begin{align}
    R(M) \leq B(A_0 + 4R_0) + B_2 + 4dB_{*}L \log{\Big(1+\frac{T}{\lambda}\Big)} + 2 B_* \log{T} + 5.\label{eq:0002}
\end{align}
And by the initial definition of $R(M)$, we have
\begin{align*}
    C_M & = R(M) + KV^{*}(s_{\text{init}}) \\
    & \leq B(A_0 + 4R_0) + B_2 + 4dB_{*}L \log{\Big(1+\frac{T}{\lambda}\Big)} + 2 B_* \log{T} + 5 + KV^{*}(s_{\text{init}})\\
    & \leq B_2 + 4dB_{*}L \log{\Big(1+\frac{T}{\lambda}\Big)} + 2 B_* \log{T} + 5 + KV^{*}(s_{\text{init}})\\
    & \qquad + 216B\big(\max\{8 \hat{\beta}_T\sqrt{d\iota},2\sqrt{2\zeta}\}\big)^2\\
    & \qquad + 12B\big(8d\iota + 8\hat{\beta}_T \gamma^2 d\iota + \zeta
    + 8\sqrt{d\iota}\hat{\beta}_T
    \sqrt{\sum_{m \in \cM_0(M)}\sum_{h=1}^{H_m}\alpha_{t(m,h)}^2 + |\cM_0|}\big)\\
    & \qquad + 8\sqrt{2}B\big( 2\max\{8 \hat{\beta}_T\sqrt{d\iota},2\sqrt{2\zeta}\}\big)\sqrt{\frac{B_2}{B}}\\
    & \qquad + 8\sqrt{2}B\big( 2\max\{8 \hat{\beta}_T\sqrt{d\iota},2\sqrt{2\zeta}\}\big)\sqrt{\frac{C_M}{B}}\\
    &\leq KV^{*}(s_{\text{init}})
    + \tilde{\cO}(dB\hat{\beta}_T^2 \iota) + \tilde{\cO}(\sqrt{B}\hat{\beta}_T\sqrt{d\iota}\sqrt{C_M}),
\end{align*}
where the first inequality holds due to \eqref{eq:0002}, the second inequality holds due to \eqref{eq:A0}, the last inequality holds due to Lemmas \ref{lemma:J} and \ref{lemma:B2} with the choice of parameter, $\alpha_t = {1}/t^2$ and $\gamma = d^{-{1}/{4}}$.
Using the fact that $x \leq a\sqrt{x} + b + z \Rightarrow x \leq \frac{1}{2}a^2 + b + z + a\sqrt{b + a^2 + z}$, we have,
\begin{align}
    \notag R(M) &= C_M - KV^{*}(s_{\text{init}})\\\notag
    & \leq \tilde{\cO}(d\iota B\hat{\beta}_T^{2}) + \tilde{\cO}(\sqrt{B}\hat{\beta}_T\sqrt{d\iota}) * \sqrt{KV^{*}(s_{\text{init}})
    + \tilde{\cO}(d\iota B\hat{\beta}_T^2 )}\\\label{R}
    &\leq \tilde{\cO}(d\iota B\hat{\beta}_T^{2}) + \tilde{\cO}(\sqrt{d\iota}B \hat{\beta}_T\sqrt{K}).
\end{align}
Here $\tilde \cO$ hides some logarithmic factors of $T$, $K$, $B$, $1/c_{\text{min}}$ and $1/ \delta$.
In addition, we have 
    \begin{align*}
        c_{\text{min}} T &\leq C_M = R(M) + K V^*(s_{\text{init}})\\
        & \leq \tilde{\cO}(d\iota B\hat{\beta}_T^{2}) + \tilde{\cO}(\sqrt{d\iota}B \hat{\beta}_T\sqrt{K})+ K V^*(s_{\text{init}}).
        \end{align*}
Therefore, we can get an upper bound of $T$, which is
\begin{align*}
    T \leq \tilde \cO\left(\frac{d^2B + dB\sqrt{K} + KB} {c_{\text{min}}}\right).
\end{align*}
Here $\tilde \cO$ hides some logarithmic factors of $K$, $B$, $1/c_{\text{min}}$ and $1/ \delta$.
Putting the upper bound of $T$ into \eqref{R}, we finish the proof of Theorem \ref{Theorem:main}.
\end{proof}

\begin{proof}[Proof of Theorem \ref{thm:pertur}]
The optimal policy will not change after we add the perturbation to the cost function. For every step, the perturbed cost function is $\rho$ larger than the original one, thus for the optimal value function after the perturbation, we have $V_\rho^* = V^* + \rho T_*$. In addition, $B + \rho T_*$  
 can be an upper bound of the perturbed value function of optimal policy $V^*$. Furthermore, the perturbed cost function has a lower bound $\rho$. Therefore, we can use the algorithm and results for the perturbed cost function and
    the regret can be written in the following form,
\begin{align*}
    R_K &= \sum_{k = 1}^K \sum_{i = 1}^{I_k} c_{k,i} - K V^*(s_{\text{init}})\\
    & \leq \sum_{k = 1}^K \sum_{i = 1}^{I_k} c^{\rho}_{k,i} - K V_\rho^*(s_{\text{init}}) + \rho T_* K,
\end{align*}
where we use the inequality $c(\cdot,\cdot) \leq c^{\rho}(\cdot,\cdot)$.
Then using Theorem \ref{Theorem:main} with $c_{\text{min}} = \rho$, and the upper bound of the optimal value function $B_\rho = B + \rho T_*$, we have the inequality below,
\begin{align*}
    R_K & \leq \tilde{\cO}(d (B+\rho T_*)\hat{\beta}_T^{2} + \sqrt{d}(B+\rho T_*) \hat{\beta}_T\sqrt{K}) + \rho T_* K\\
    &= \tilde{\cO}(d B\hat{\beta}_T^{2} + \sqrt{d}B \hat{\beta}_T\sqrt{K}).
\end{align*}
Here we use the choice of parameter $\rho = (T_*K)^{-1}$ and the $\tilde \cO$ hides some logarithmic term of $T_*$, $K$, $B$, $1/\delta$.
\end{proof}

\section{Proof of Lemma \ref{LEMMA:VARIANCE ESTIMATOR}}
We first prove an upper bound of the error between the variance estimator and the true variance. 

\begin{lemma}\label{lemma:Variance}
    
For any $t$, $j=j(t)$ as the index of the value functions $V$ at step $t$, and $l \in [L]$, let $V_j, \hat{\boldsymbol{\theta}}_{t,l}, \sigma_t,  \hat{\bSigma}_{t,l}$ be defined in Algorithms \ref{algorithm1} and \ref{algorithm3}.  We have the following inequality,
\begin{align*}
\left|\bar{\mathbb{V}}_{t,l} V_j^{2^l}\left(s_t, a_t\right)-[\mathbb{V} V_j^{2^l}]\left(s_t, a_t\right)\right| &\leq \min \left\{B^{2^{l+1}},\left\|\hat{\boldsymbol{\Sigma}}_{j,l+1}^{1 / 2}\left(\boldsymbol{\theta}^*-\hat{\boldsymbol{\theta}}_{t,l+1}\right)\right\|_2\left\|\hat{\boldsymbol{\Sigma}}_{j,l+1}^{-1 / 2} \boldsymbol{\phi}_{V_j^{2^{l+1}}}\left(s_t, a_t\right)\right\|_2 \right\} \\
& \qquad +\min \left\{B^{2^{l+1}}, 2 B^{2^l}\left\|\hat{\boldsymbol{\Sigma}}_{j,l}^{1 / 2}\left(\boldsymbol{\theta}^*-\widehat{\boldsymbol{\theta}}_{t,l}\right)\right\|_2\left\|\hat{\boldsymbol{\Sigma}}_{j,l}^{-1 / 2} \boldsymbol{\phi}_{V_j^{2^l}}\left(s_t, a_t\right)\right\|_2\right\},
\end{align*}
if the inequality $\left|V_j(s)\right| \leq B$ holds for all state $s\in \cS$.
\end{lemma}
\begin{proof}[Proof of Lemma \ref{lemma:Variance}]
We first substitute the definition of variance estimator (Line \ref{algorithm3:line2} of Algorithm \ref{algorithm3}) into the left-hand side, and the error between the variance estimator and the true variance can be bounded by
\begin{align*}
& \left|[\bar{\mathbb{V}}_{t,l} V_j^{2^l}]\left(s_t, a_t\right)-[\mathbb{V} V_j^{2^l}]\left(s_t, a_t\right)\right| \\
& =\bigg|\left[\left\langle\boldsymbol{\phi}_{V_j^{2^{l+1}}}\left(s_t, a_t\right), \hat{\boldsymbol{\theta}}_{t,l+1}\right\rangle\right]_{\left[0, B^{2^{l+1}}\right]}-\left\langle\boldsymbol{\phi}_{V_j^{2^{l+1}}}\left(s_t, a_t\right), \boldsymbol{\theta}^*\right\rangle\\
& \qquad +\left\langle\boldsymbol{\phi}_{V_j^{2^l}}\left(s_t, a_t\right), \boldsymbol{\theta}^*\right\rangle^2-\left[\left\langle\boldsymbol{\phi}_{V_j^{2^l}}\left(s_t, a_t\right), \widehat{\boldsymbol{\theta}}_{t,l}\right\rangle\right]_{[0, B^{2^l}]}^2\bigg| \\
& \leq \underbrace{\left|\left[\left\langle\boldsymbol{\phi}_{V_j^{2^{l+1}}}\left(s_t, a_t\right), \hat{\boldsymbol{\theta}}_{t,l+1}\right\rangle\right]_{\left[0, B^{2^{l+1}}\right]}-\left\langle\boldsymbol{\phi}_{V_j^{2^{l+1}}}\left(s_t, a_t\right), \boldsymbol{\theta}^*\right\rangle\right|}_{I_1}\\
& \qquad +\underbrace{\left|\left\langle\boldsymbol{\phi}_{V_j^{2^l}}\left(s_t, a_t\right), \boldsymbol{\theta}^*\right\rangle^2-\left[\left\langle\boldsymbol{\phi}_{V_j^{2^l}}\left(s_t, a_t\right), \widehat{\boldsymbol{\theta}}_{t,l}\right\rangle\right]_{[0, B^{2^l}]}^2\right|}_{I_2},
\end{align*}
where the inequality holds due to the triangle inequality.
To bound the term $I_1$, we have
\begin{align*}
I_1 \leq\left|\left\langle\boldsymbol{\phi}_{V_j^{2^{l+1}}}\left(s_t, a_t\right), \hat{\boldsymbol{\theta}}_{t,l+1}-\boldsymbol{\theta}^*\right\rangle\right| \leq\left\|\hat{\boldsymbol{\Sigma}}_{j,l+1}^{1 / 2}\left(\boldsymbol{\theta}^*-\hat{\boldsymbol{\theta}}_{t,l+1}\right)\right\|_2\left\|\hat{\boldsymbol{\Sigma}}_{j,l+1}^{-1 / 2} \boldsymbol{\phi}_{V_j^{2^{l+1}}}\left(s_t, a_t\right)\right\|_2,
\end{align*}
The first inequality holds because both terms are in $\left[0, B^{2^{l+1}}\right]$, which implied by the assumption of $\left|V_j(s)\right| \leq B$ . The second inequality holds due to Cauchy-Schwarz inequality. Furthermore, the facts that $\left|V_j(s)\right| \leq B$ and both terms in $I_1$ lie in  $\left[0, B^{2^{l+1}}\right]$ suggest that $I_1 \leq B^{2^{l+1}}$. Combining these two upper bounds, we have
\begin{align*}
I_1 \leq \min \left\{B^{2^{l+1}},\left\|\hat{\boldsymbol{\Sigma}}_{j,l+1}^{1 / 2}\left(\boldsymbol{\theta}^*-\hat{\boldsymbol{\theta}}_{t,l+1}\right)\right\|_2\left\|\hat{\boldsymbol{\Sigma}}_{j,l+1}^{-1 / 2} \boldsymbol{\phi}_{V_j^{2^{l+1}}}\left(s_t, a_t\right)\right\|_2 \right\} .
\end{align*}
To bound the term $I_2$, we have
\begin{align*}
I_2 & =\left|\left\langle\boldsymbol{\phi}_{V_j^{2^l}}\left(s_t, a_t\right), \boldsymbol{\theta}^*\right\rangle-\left[\left\langle\boldsymbol{\phi}_{V_j^{2^l}}\left(s_t, a_t\right), \widehat{\boldsymbol{\theta}}_{t,l}\right\rangle\right]_{[0, B^{2^l}]}\right| \\
& \qquad \cdot\left|\left\langle\boldsymbol{\phi}_{V_j^{2^l}}\left(s_t, a_t\right), \boldsymbol{\theta}^*\right\rangle+\left[\left\langle\boldsymbol{\phi}_{V_j^{2^l}}\left(s_t, a_t\right), \widehat{\boldsymbol{\theta}}_{t,l}\right\rangle\right]_{[0, B^{2^l}]}\right| \\
& \leq 2 B^{2^l}\left|\left\langle\boldsymbol{\phi}_{V_j^{2^l}}\left(s_t, a_t\right), \boldsymbol{\theta}^*-\widehat{\boldsymbol{\theta}}_{t,l}\right\rangle\right| \\
& \leq 2 B^{2^l}\left\|\hat{\boldsymbol{\Sigma}}_{j,l}^{1 / 2}\left(\boldsymbol{\theta}^*-\widehat{\boldsymbol{\theta}}_{t,l}\right)\right\|_2\left\|\hat{\boldsymbol{\Sigma}}_{j,l}^{-1 / 2} \boldsymbol{\phi}_{V_j^{2^l}}\left(s_t, a_t\right)\right\|_2,
\end{align*}
where the first inequality holds because $|V_j(s)| \leq B $ and thus both terms in the second absolute value   are bounded by $B^{2^l}$, and the second inequality holds by Cauchy-Schwarz inequality. Furthermore, both terms in $I_2$ lie in $[0,B^{2^{l+1}}]$. Combining these two upper bounds, we have

\begin{align*}
I_2 \leq \min \left\{B^{2^{l+1}}, 2 B^{2^l}\left\|\hat{\boldsymbol{\Sigma}}_{j,l}^{1 / 2}\left(\boldsymbol{\theta}^*-\widehat{\boldsymbol{\theta}}_{t,l}\right)\right\|_2\left\|\hat{\boldsymbol{\Sigma}}_{j,l}^{-1 / 2} \boldsymbol{\phi}_{V_j^{2^l}}\left(s_t, a_t\right)\right\|_2\right\}.
\end{align*}
Thus, we finish the proof of Lemma \ref{lemma:Variance} by combining the bounds of terms $I_1$ and $I_2$.
\end{proof}
\begin{proof}[Proof of Lemma \ref{LEMMA:VARIANCE ESTIMATOR}]
\textbf{Special case: $t = 1$} 

When $t=1$, we have the index of interval $j = 1$. In this interval, $V_0$ is generated from our initialization instead of the DEVI process.
 For each $l \in [L-1]$, let $\xb_{t,l} = \bar{\sigma}_{t,l}^{-1}\bphi_{V_j^{2^l}}$,
$\eta_{t,l} = \bar{\sigma}_{t,l}^{-1}\big[V_j^{2^l}(s_{t+1}) -
\la \bphi_{V_j^{2^l}},\btheta^{*}\ra\big] = 0$, $\forall l \in [L]$.   
$\bmu^{*} = \btheta^{*}$, $y_{t,l} = \big\la \bmu^*,\xb_{t,l}\big\ra + \eta_{t,l}$. 
$\Zb_{t,l} = \lambda \Ib + \sum_{t^{\prime} = 1}^{t} \xb_{t^{\prime},l}\xb_{t^{\prime},l}^{\top}$,
$\bbb_{t,l} = \sum_{t^{\prime} = 1}^{t} \xb_{t^{\prime},l}y_{t^{\prime},l}$ and $\bmu_{t,l} = \Zb_{t,l}^{-1}\bbb_{t,l} $.
Then the following inequalities
hold, 
\begin{align}
\EE\big[\eta_{t,l}\big|\cG_{t,l}\big] = 0, |\eta_{t,l}| = 0,
\big|\eta_{t,l}\min \{1,\|\xb_{t,l}\|_{\Zb_{t-1}^{-1}} \}\big| = 0,
\|\xb_{t,l}\| \leq \bar{\sigma}_{t,l}^{-1} \cdot 1 \leq 1/(\alpha_t B^{2^l}),\EE\big[\eta_{t,l}^2\big|\cG_{t,l}\big] = 0.\label{Berstein condition1}
\end{align}
Using Theorem \ref{thm:Bernstein} with \eqref{Berstein condition1},
we can get $\left\|\boldsymbol{\mu}_{t,l}-\boldsymbol{\mu}^*\right\|_{\mathbf{Z}_{t,l}} \leq \sqrt{\lambda}\left\|\boldsymbol{\mu}^*\right\|_2$
with probability at least $1-\delta/2$. After taking an union bound over all level $l\in [L]$,  $\btheta^* \in \hat{C}_{1}$ holds with probability at least $1 - L\delta/2$.

At the end of step $t=1$, the step number $t$ will be doubled and DEVI will be triggered and output $V_1$. We next show that conditioned on the high-probability event $\{\btheta^* \in \hat{C}_{1,0}\}$, the output of DEVI will satisfy $V_1 \leq V^*$ (optimism).

We prove this by induction. First, the initialization $Q^{(0)}$ and $V^{(i)} $ are equal to $0$, thus we have $Q^{(0)} \leq Q^*$ and $V^{(0)} \leq V^*$. For the induction hypothesis, suppose $Q^{(i)} \leq Q^*$ and $V^{(i)} \leq V^*$ for some $i$, we want to show $Q^{(i+1)} \leq Q^*$ and $V^{(i+1)} \leq V^*$. To prove this, by the iteration principle of $Q^{(i)}$, we have
\begin{align}
    Q^{(i+1)}(\cdot,\cdot) &=
    c(\cdot,\cdot) + (1-q) \cdot \min_{\btheta \in \hat{\cC}_{1,0} \cap \cB}\big\la \btheta, \bphi_{V^{(i)}}(\cdot,\cdot)\big\ra  \notag\\
    & \leq c(\cdot,\cdot) + (1-q) \cdot \PP V^{(i)}(\cdot,\cdot)\label{eq:DEVI}\\
    & \leq Q^{(i)} \notag\\
    & \leq Q^*\notag,
\end{align}
where the first inequality is because we are taking a minimum on a set containing $\btheta^*$, the second inequality is by the Bellman equation and the last inequality is by our induction hypothesis.

\textbf{Initial Step:} We then go on to the next interval. From now on, the value function will be the output of DEVI. When the interval number $j = 1$ and $t \in [t_1 + 1,t_2]$ and we suppose the high-probability event $\{\btheta^* \in \hat{C}_j\}$ occurs, we have $V_1 \leq V^* \leq B$. \\
 For each $l \in [L-1]$, let $\xb_{t,l} = \bar{\sigma}_{t,l}^{-1}\bphi_{V_j^{2^l}}$
$\eta_{t,l} = \bar{\sigma}_{t,l}^{-1}\ind\big\{\btheta^* \in \hat{\cC}_{j,l}\cap \hat{\cC}_{j,l+1}\big\}\big[V_j^{2^l}(s_{t+1}) -
\la \bphi_{V_j^{2^l}},\btheta^{*}\ra\big]$ for $l \in [L-1]$.   
$\bmu^{*} = \btheta^{*}$, $y_{t,l} = \big\la \bmu^*,\xb_{t,l}\big\ra + \eta_{t,l}$. 
$\Zb_{t,l} = \lambda \Ib + \sum_{t^{\prime} = 1}^{t} \xb_{t^{\prime},l}\xb_{t^{\prime},l}^{\top}$,
$\bbb_{t,l} = \sum_{t^{\prime} = 1}^{t} \xb_{t^{\prime},l}y_{t^{\prime},l}$ and $\bmu_{t,l} = \Zb_{t,l}^{-1}\bbb_{t,l} $.
Then the following inequalities hold, 
\begin{align}
\label{Bernstein condition2.1}
\notag &\EE\big[\eta_{t,l}\big|\cG_{t,l}\big] = 0, |\eta_{t,l}| \leq \bar{\sigma}_{t,l}^{-1} B^{2^l} \leq 
1/\alpha_t, \|\xb_{t,l}\|_2 \leq \bar{\sigma}_{t,l}^{-1} \cdot B^{2^l} \leq 1/\alpha_t,\\
&\big|\eta_{t,l}\min \{1,\|\xb_{t,l}\|_{\Zb_{t-1}^{-1}} \}\big| \leq 
\bar{\sigma}_{t,l}^{-2}\| \bphi_{V_j^{2^l}}\|_{\Zb_{t-1}^{-1}}B^{2^l} \leq 
1/\gamma^2.
\end{align}
We also have
\begin{align}
    \notag
    \EE\big[\eta_{t,l}^2\big|\cG_{t,l}\big] &= \bar{\sigma}_{t,l}^{-2} \ind\big\{\btheta^* \in \hat{\cC}_{j,l}\cup \hat{\cC}_{j,l+1}\big\}\big[\VV V_j^{2^l}\big](s_t,a_t)\\\notag
    & \leq \bar{\sigma}_{t,l}^{-2} \ind\big\{\btheta^* \in \hat{\cC}_{j,l}\cup \hat{\cC}_{j,l+1}\big\}\bigg(\big[\VV_{t,l} V_j^{2^l}\big](s_t,a_t) \\\notag
    & \qquad + \min \Big\{B^{2^{l+1}},2B^{2^{l}}\Big\|\hat{\bSigma}_{j,l}^{-\frac{1}{2}}\bphi_{V_j^{2^l}}(s_t,a_t) \Big\|_2 \cdot\Big\|\hat{\bSigma}_{j,l}^{\frac{1}{2}}(\btheta^*-\hat{\btheta}_{t_j,l}) \Big\|_2\Big\}\\\notag
    & \qquad + \min \Big\{B^{2^{l+1}},\Big\|\hat{\bSigma}_{j,l+1}^{-\frac{1}{2}}\bphi_{V_j^{2^{l+1}}}(s_t,a_t) \Big\|_2\cdot\Big\|\hat{\bSigma}_{j,l+1}^{\frac{1}{2}}(\btheta^*-\hat{\btheta}_{t_j,l}) \Big\|_2\Big\}\bigg)\\\notag 
    &\leq \bar{\sigma}_{t,l}^{-2}\Big(\big[\VV_{t,l}V_j^{2^l}\big](s_t,a_t) + B^{2^{l+1}}E_{t,l}\Big)\\ &\leq 1 \label{Bernstein condition2.2},
\end{align}
where the first equation holds by the definition of $\eta_{t,l}$, the second inequality holds by Lemma \ref{lemma:Variance}, and the third inequality holds by the definition of our confidence ellipsoid $\hat{C}_{j,l}$, $\hat{C}_{j,l+1}$ and $E_{t,l}$. The last inequality holds due to the definition of $\bar{\sigma}_{t,l}^{-2}$. 

For $l = L-1$, 
let $\xb_{t,L-1} = \bar{\sigma}_{t,L-1}^{-1}\bphi_{V_j^{2^{L-1}}}$,
$\eta_{t,L-1} = \bar{\sigma}_{t,L-1}^{-1}\big[V_j^{2^{L-1}}(s_{t+1}) -
\la \bphi_{V_j^{2^{L-1}}},\btheta^{*}\ra\big]$,     
$\bmu^{*} = \btheta^{*}$, $y_{t,L-1} = \big\la \bmu^*,\bx_{t,L-1}\big\ra + \eta_{t,L-1}$, 
$\Zb_{t,L-1} = \lambda \Ib + \sum_{t^{\prime} = 1}^{t} \xb_{t^{\prime},L-1}\xb_{t^{\prime},L-1}^{\top}$,
$\bbb_{t,L-1} = \sum_{t^{\prime} = 1}^{t} \xb_{t^{\prime},L-1}y_{t^{\prime},L-1}$ and $\bmu_{t,L-1} = \Zb_{t,L-1}^{-1}\bbb_{t,L-1} $.
Then the following inequalities hold, 
\begin{align}\label{Bernstein condition2.3}\notag
&\EE\big[\eta_{t,L-1}\big|\cG_{t,L-1}\big] = 0, |\eta_{t,L-1}| \leq 1/\alpha_t, \|\bx_{t,L-1}\| \leq \bar{\sigma}_{t,L-1}^{-1} \cdot B^{2^{L-1}} \leq 1/\alpha_t,\\
&\big|\eta_{t,L-1}\min \{1,\|\xb_{t,L-1}\|_{\Zb_{t,L-1}^{-1}} \}\big| \leq 1/\gamma^2.
\end{align}
And we can get directly by the definition of $\bar{\sigma}_{t,L-1}$ and the fact that $0 \leq V_j(s) \leq B$,
\begin{align}\label{Bernstein condition2.4}
    \EE\big[\eta_{t,l}^2\big|\cG_{t,l}\big] 
    \leq 1.
\end{align}
Using Theorem \ref{thm:Bernstein} with conditions \eqref{Bernstein condition2.1} and \eqref{Bernstein condition2.2} for $l \in [L-1]$ and \eqref{Bernstein condition2.3} and \eqref{Bernstein condition2.4} for $l = L$, choose $\epsilon = 1/\gamma^2$. With probability at least $1-\delta/({t_2(t_2+1)})$
\begin{align}
\left\|\boldsymbol{\mu}_{t,l}-\boldsymbol{\mu}^*\right\|_{\mathbf{Z}_{t,l}} \leq \beta_t+\sqrt{\lambda}\left\|\boldsymbol{\mu}^*\right\|_2, \label{eq:mu}
\end{align}
where 
\begin{align*}
\beta_t=& 12 \sqrt{ d \log \left(1+t /(d \lambda \alpha_t^2)\right) \log \left(256(\log (\gamma^2 / \alpha_t )+1) t^4 / \delta\right)} \\
&+ \frac{30}{\gamma^2}\log \left(256(\log (\gamma^2 / \alpha_t )+1) t^4 / \delta\right). 
\end{align*}
Here we use the fact that $4t^2t_2(t_2+1) \leq 32t^4$ for $t \in [t_1+1,t_2]$, since $\forall t \in [t_1+1,t_2], t_2 \leq 2t, t_2 + 1 \leq 4t$ holds due to our criteria of doubling the number of steps.

We take a union bound over all level $l\in [L]$, and then we have the result that, with probability at least $1-(L \delta)/({t_2(t_2+1)})$, \eqref{eq:mu} holds for all level $l$ simultaneously. 

\textbf{Induction step:} Suppose that, with probability at least $1-\delta^{\prime}$, inequality \eqref{eq:mu} holds for all $t \in [1,t_{j-1}]$ and all $l$. For $t \in [t_{j-1}+1,t_{j}] $, we can define $\bmu_{t,l}$ and $\Zb_{t,l}$ in the same way as the initial step. We claim 
 that with probability of at least
$1 - \delta^{\prime} - (L\delta)/({t_j(t_j+1)})$,  inequality \eqref{eq:mu} holds for all $t \in [1,t_{j}]$ and all $l$ simultaneously. Note that in the previous step, the only condition we use is the optimism $V_1 \leq V^* \leq B_*$. Using the same argument in \eqref{eq:DEVI}, we can see that assuming the event of the true parameter $\btheta^* \in  \hat{C}_{j,0}$ holds, the output $V_j$ will satisfy $V_j \leq V^*$, thus $V_j \leq V^* \leq B$. Then we can follow the proof in the initial step and see \eqref{eq:mu} holds for $t \in [t_{j-1}+1,t_j]$ and $l \in [L]$ simultaneously with probability at least $1 - (L\delta)/({t_j(t_j+1)})$. By taking a union bound, we make an induction from $t \in [1,t_{j-1}]$ to $t \in [1,t_{j}]$. 

Finally, we use induction to see that with probability at least
\begin{align*}
    1- \sum_{j=1}^{J}\frac{L\delta}{t_j(t_j+1)} = 1-\sum_{j=1}^{J}L\delta\big(\frac{1}{t_j}-\frac{1}{t_j+1}\big) \leq 1-L\delta,
\end{align*}
\eqref{eq:mu} holds for all $t$ and $l \in [L]$ simultaneously. 

We define $\cE = \cup_{t,l} \big\{\left\|\boldsymbol{\mu}_{t,l}-\boldsymbol{\mu}^*\right\|_{\Zb_{t,l}} \leq \beta_t+\sqrt{\lambda}\left\|\boldsymbol{\mu}^*\right\|_2\big\}$, and we have shown that $\cE$ holds with probability of at least $1 - L \delta$. 
Conditioned on the event $\cE$, recalling our definition of $\hat{\cC}_{j,l} = \Big\{ \btheta: \big\| \hat{\bSigma}_{t_j,l}^{\frac{1}{2}}(\hat{\btheta}_{t_j,l}-\btheta)\big\|_2 \leq \hat{\beta}_{t_j}\Big\}$, 
we have the following results. 

For $t = 1$, $l \in [L]$, the definition of $\bmu_{1,l}$ is just the same as $\hat{\btheta}_{1,l}$ and $\hat{\bSigma}_{t_j,l} = \Zb_{t,l}$ when $j = 1$ and $t = 1$.
Thus, we have
    $\btheta^* \in \hat{\cC}_{1,l}, \forall  l \in [L].$
Then we consider the term with the highest order. For all $j$ and $l = L-1$, we directly have $\bmu_{t_j,L-1} = \hat{\btheta}_{j,L-1}$. Thus, we have
    $\btheta^* \in \hat{\cC}_{j,L-1}$.
For all $j$ and $l \in [L-1]$, we have the following induction argument, 
\begin{align*}
    & \btheta^* \in \hat{\cC}_{j,l}\cap \hat{\cC}_{j,l+1} \Rightarrow \ind\{\btheta^* \in \hat{\cC}_{j,l}\cap \hat{\cC}_{j,l+1}\}=1 \Rightarrow 
    \bmu_{t,l} = \hat{\btheta}_{t,l},\Zb_{t,l} = \tilde{\bSigma}_{t,l} \Rightarrow
    \btheta^* \in \hat{\cC}_{j+1,l}.
\end{align*}
By induction on $l$ and $j$, conditioned on event $\cE$, we have \begin{align}
        \btheta^{*} \in \hat{\cC}_{j,l}, \quad 0 \leq Q_j(\cdot,\cdot) \leq Q^{*}(\cdot,\cdot), 
    \end{align}
and according to Lemma \ref{lemma:Variance} and the definition of the confidence ellipsoid of $\hat{C}_{j,l}$, we have
\begin{align*}\Big|\big[\bar{\VV}_{t,l}V_{j}^{2^l}](s_t,a_t)-\big[\VV V_{j}^{2^l}\big](s_t,a_t)\Big| \leq B^{2^{l+1}}E_{t,l}.
\end{align*}
\end{proof}
\section{Proof of Lemma \ref{LEMMA:E1}}
\begin{proof}[Proof of Lemma \ref{LEMMA:E1}]
    According to the DEVI algorithm, the output $Q$ can be denoted by some iteration of $Q^{(n)}$, i.e.
\begin{align}
    Q_{j_m}(\cdot,\cdot) &= Q^{(n)}(\cdot,\cdot) \quad \text{for some iteration } n\in \NN \notag \\
    V_{j_m} (\cdot)&= \min_{a \in \cA} Q^{(n)}(\cdot,a) = V^{(n)}(\cdot).\notag
\end{align}
Through the design of the DEVI algorithm,
\begin{align}
    Q^{(n)}(s_{m,h},a_{m,h}) &= c_{m,h} + (1-q)\cdot \min_{\btheta \in \cC_{j_m,0} \cap \cB}
    \big\la\btheta, \bphi_{V^{(n-1)}}(s_{m,h},a_{m,h})\big\ra \notag\\
    &= c_{m,h} + (1-q)\cdot 
    \big\la\btheta_{m,h}, \bphi_{V^{(n-1)}}(s_{m,h},a_{m,h})\big\ra \notag\\
    &= c_{m,h} + (1-q)\cdot 
    \big\la\btheta_{m,h}, \bphi_{V^{(n)}}(s_{m,h},a_{m,h})\big\ra \notag\\
    & \qquad + (1-q) \cdot 
    \big\la\btheta_{m,h}, \big[\bphi_{V^{(n-1)}} -\bphi_{V^{(n)}}\big](s_{m,h},a_{m,h})\big\ra\notag\\
    & \geq c_{m,h} + (1-q)\cdot 
    \big\la \btheta_{m,h}, \bphi_{V^{(n)}}(s_{m,h},a_{m,h})\big\ra - (1-q)\frac{1}{t_{j_m}}, \label{eq:2}
\end{align}
where $\btheta_{m,h} = \argmin_{\btheta \in \cC_{j_m,0} \cap \cB} \big\la\btheta, \bphi_{V^{(n-1)}}(s_{m,h},a_{m,h})\big\ra$. The last inequality is because of the terminal condition of our DEVI algorithm.

With a similar argument, 
\begin{align}
    Q^{(n)}(s_{m,h},a_{m,h}) \leq c_{m,h} + (1-q)
    \big\la \btheta_{m,h}, \bphi_{V^{(l)}}(s_{m,h},a_{m,h})\big\ra + (1-q)\frac{1}{t_{j_m}}. \label{eq:3}
\end{align}
According to inequality \eqref{eq:2} , we have
\begin{align*}
        &c_{m,h} + \PP V_{j_m}(s_{m,h},a_{m,h}) - Q_{j_m}(s_{m,h},a_{m,h})\\
        &\leq c_{m,h} + \PP V_{j_m}(s_{m,h},a_{m,h})
        -\big[c_{m,h} + (1-q)\cdot 
    \big\la \btheta_{m,h}, \bphi_{V_{j_m}}(s_{m,h},a_{m,h})\big\ra - (1-q)\frac{1}{t_{j_m}} \big]\\
        &\leq \big\la \btheta^{*}-\btheta_{m,h},\bphi_{V_{j_m}}(s_{m,h},a_{m,h})\big\ra  + 
        q \cdot \big\la \btheta_{m,h}, \bphi_{V_{j_m}}(s_{m,h},a_{m,h})\big\ra+ \frac{1 - q}{t_{j_m}}\\
        & \leq \big\la \btheta^{*}-\btheta_{m,h},\bphi_{V_{j_m}}(s_{m,h},a_{m,h})\big\ra + \frac{B_* + 1 - q}{t_{j_m}}\\
        & \leq \big\la \btheta^{*}-\btheta_{m,h},\bphi_{V_{j_m}}(s_{m,h},a_{m,h})\big\ra + \frac{B_* + 1}{t_{j_m}},
\end{align*}
where the second inequality holds due to the definition $\PP V_{j_m}(s_{m,h},a_{m,h}) = \big\la \btheta^*,\bphi_{V_{j_m}}(s_{m,h},a_{m,h})\big\ra$ and the result $V_{j_m} \leq V^* \leq B_*$ proved in Lemma \ref{LEMMA:VARIANCE ESTIMATOR} and the third inequality holds because of our choice of parameter $q$ in Algorithm \ref{algorithm1}.

 We also bound the negative part, so we can get an upper bound of its absolute value, which is
\begin{align*}
        &-\big[c_{m,h} + \PP V_{j_m}(s_{m,h},a_{m,h}) - Q_{j_m}(s_{m,h},a_{m,h})\big]\\
        &\leq -c_{m,h} - \PP V_{j_m}(s_{m,h},a_{m,h})
        +\big[c_{m,h} + (1-q)\cdot 
    \big\la \btheta_{m,h}, \bphi_{V_{j_m}}(s_{m,h},a_{m,h})\big\ra + (1-q)\frac{1}{t_{j_m}} \big]\\
        &\leq \big\la \btheta_{m,h}-\btheta^{*},\bphi_{V_{j_m}}(s_{m,h},a_{m,h})\big\ra  - 
        q \cdot \big\la \btheta_{m,h}, \bphi_{V_{j_m}}(s_{m,h},a_{m,h})\big\ra+ \frac{1 - q}{t_{j_m}}\\
        & \leq \big\la \btheta_{m,h}-\btheta^{*},\bphi_{V_{j_m}}(s_{m,h},a_{m,h})\big\ra + \frac{B_* + 1 - q}{t_{j_m}}\\
        & \leq \big\la \btheta_{m,h}-\btheta^{*},\bphi_{V_{j_m}}(s_{m,h},a_{m,h})\big\ra + \frac{B_* + 1}{t_{j_m}},
\end{align*}
where the first inequality holds because of \eqref{eq:3}, the second inequality is due to the definition $\PP V_{j_m}(s_{m,h},a_{m,h}) = \big\la \btheta^*,\bphi_{V_{j_m}}(s_{m,h},a_{m,h})\big\ra$ and $V_{j_m} \leq V^* \leq B_*$ by Lemma \ref{LEMMA:VARIANCE ESTIMATOR} and the third inequality holds because of our choice of parameter $q$ in algorithm \ref{algorithm1}.
Thus we can the upper bound of the absolute value:
    \begin{align}
        \big|c_{m,h} + \PP V_{j_m}(s_{m,h},a_{m,h}) - Q_{j_m}(s_{m,h},a_{m,h})\big| \leq \big|\big\la \btheta_{m,h}-\btheta^{*},\bphi_{V_{j_m}}(s_{m,h},a_{m,h})\big\ra \big| + \frac{B_* + 1}{t_{j_m}}.\label{eq:4}
    \end{align}
To bound the first item of the right-hand side of \eqref{eq:4}, we have
\begin{align*}
    \Big|\big\la \btheta_{m,h}-\btheta^{*},\bphi_{V_{j_m}}(s_{m,h},a_{m,h})\big\ra \Big| &\leq \Big(\big\|\btheta^{*} - \hat{\btheta}_{j_m}\big\|_{\tilde{\bSigma}_{t,0}}+\big\|\btheta_{m,h} - \hat{\btheta}_{j_m}\big\|_{\tilde{\bSigma}_{t,0}}\Big)
    \big\|\bphi_{V_{j_m}}(s_{m,h},a_{m,h})\big\|_{\tilde{\bSigma}_{t,0}^{-1}}\\
    &\leq 2\Big(\big\|\btheta^{*} - \hat{\btheta}_{j_m}\big\|_{\hat{\bSigma}_{j_m,0}}+\big\|\btheta_{m,h} - \hat{\btheta}_{j_m}\big\|_{\hat{\bSigma}_{j_m,0}}\Big)
    \big\|\bphi_{V_{j_m}}(s_{m,h},a_{m,h})\big\|_{\tilde{\bSigma}_{t,0}^{-1}}\\
    &\leq 4\hat{\beta}_{T}\big\|\bphi_{V_{j_m}}(s_{m,h},a_{m,h})\big\|_{\tilde{\bSigma}_{t,0}^{-1}}.
\end{align*}
The first inequality uses the triangle inequality. The second inequality uses our partition of intervals. Since the condition is not triggered, we have $\det(\tilde{\bSigma}_{t,0}) \leq 2 \det(\hat{\bSigma}_{j_m,0})$. The third inequality holds because the high-probability event in \ref{LEMMA:VARIANCE ESTIMATOR} shows $\btheta_{m,h}$ lies in $\hat{\cC}_{j_m,0}$ and Lemma \ref{LEMMA:VARIANCE ESTIMATOR}.

Meanwhile, due to the fact that $0 \leq V_{j_m} \leq B_*$,
\begin{align*}
    \Big|\big\la \btheta_{m,h}-\btheta^{*},\bphi_{V_{j_m}}(s_{m,h},a_{m,h})\big\ra \Big| &\leq B_*.
\end{align*}
Combining these two upper bounds, we have
\begin{align}
    \Big|\big\la \btheta_{m,h}-\btheta^{*},\bphi_{V_{j_m}}(s_{m,h},a_{m,h})\big\ra \Big| &\leq \min \Big\{B_*,4\hat{\beta}_{T}\big\|\bphi_{V_{j_m}}(s_{m,h},a_{m,h})\big\|_{\tilde{\bSigma}_{t,0}^{-1}}\Big\}\label{eq:5}.
\end{align}

Combining all of these, we can finish the proof of lemma \ref{LEMMA:E1}, which is
\begin{align*}
    &\sum_{m \in \cM_0(M)} \sum_{h=1}^{H_m}\big[c_{m,h} + \PP V_{j_m}(s_{m,h},a_{m,h}) - V_{j_m}(s_{m,h})\big] \\
    &\leq \sum_{m \in \cM_0(M)} \sum_{h=1}^{H_m}\Big|\big\la \btheta_{m,h}-\btheta^{*},\bphi_{V_{j_m}}(s_{m,h},a_{m,h})\big\ra \Big| + \sum_{m \in \cM_0(M)} \sum_{h=1}^{H_m}\frac{B_* +1}{t_{j_m}}\\
    &\leq \sum_{m \in \cM_0(M)}\sum_{h=1}^{H_m}\min \Big\{B_*,4\hat{\beta}_{T}\big\|\bphi_{V_{j_m}}(s_{m,h},a_{m,h})\big\|_{\tilde{\bSigma}_{t,0}^{-1}}\Big\} + \sum_{m \in \cM_0(M)}\sum_{h=1}^{H_m}\frac{B_*+1}{t_{j_m}},\\
\end{align*}
where the first inequality is from \eqref{eq:4} and the second is from \eqref{eq:5}.
\end{proof}

\section{Proof of Other Technical Lemmas}
\begin{proof}[Proof of Lemma \ref{LEMMA:REGRET DECOMPOSITION}]
The regret can be written as
    \begin{align*}
        R(M) &\leq 
        \sum_{m=1}^M \sum_{h=1}^{H_m} c_{m,h} - \sum_{m \in \cM(M)} V_{j_m}(s_{\text{init}}) + 1\\
    & = \sum_{m=1}^M \sum_{h=1}^{H_m} c_{m,h} 
    + \sum_{m=1}^M \sum_{h=1}^{H_m}\big[V_{j_m}(s_{m,h+1}) -  V_{j_m}(s_{m,h})\big]\\
    &\qquad + \sum_{m=1}^M \sum_{h=1}^{H_m} \big[ V_{j_m}(s_{m,h})-V_{j_m}(s_{m,h+1}) \big] - \sum_{m \in \cM(m)} V_{j_m}(s_{\text{init}}) + 1\\
    & = 
        \sum_{m=1}^M \sum_{h=1}^{H_m}\big[c_{m,h} + \PP V_{j_m}(s_{m,h},a_{m,h}) - V_{j_m}(s_{m,h})\big] \\
    & \qquad + \sum_{m=1}^M \sum_{h=1}^{H_m}\big[V_{j_m}(s_{m,h+1}) - \PP V_{j_m}(s_{m,h},a_{m,h})\big]\\
    &\qquad + \sum_{m=1}^M \sum_{h=1}^{H_m} \big[ V_{j_m}(s_{m,h})-V_{j_m}(s_{m,h+1}) \big] - \sum_{m \in \cM(M)} V_{j_m}(s_{\text{init}}) + 1,
    \end{align*}
    where the inequality holds due to the optimism of $V_{j_m}$,i.e $V_{j_m}(s) \leq V^*(s), \forall s \in \cS$ under the event of Lemma \ref{LEMMA:VARIANCE ESTIMATOR}.
\end{proof}
\begin{proof}[Proof of Lemma \ref{lemma:J}]
    We divide the calls to DEVI into two parts $J_1$ and $J_2$, 
    where $J_1$ is the total number of times that the determinant is doubled
and $J_2$ is the total number of times that the time step is doubled. We divide $J_1$ into $J_1 = \sum_{l = 0}^{L-1}J_{1,l}$, where $J_{1,l}$ is the total number of times that the determinant of moment order $l$ is doubled. For $ \forall l \in [L]$, we then give the bound of $J_{1,l}$. 
\begin{align*}
\left\|\tilde{\boldsymbol{\Sigma}}_{T,l}\right\|_2 & =\left\|\lambda \mathbf{I}+ \sum_{t=1}^{T} \bar{\sigma}_{t,l}^{-2}\bphi_{V_j^{2^l}}\left(s_t, a_t\right) \boldsymbol{\phi}_{V_j^{2^l}}\left(s_t, a_t\right)^{\top}\right\|_2 \\
& \leq \lambda+ \sum_{t=1}^{T}\left\|\bar{\sigma}_{t,l}^{-1}\bphi_{V_j^{2^l}}\left(s_t, a_t\right)\right\|_2^2 \\
& \leq \lambda+ \frac{T}{\alpha_t^2}\\
&= \lambda + T^2,
\end{align*}
where the first inequality is by the triangle inequality and the second inequality holds by $0\leq V_j(s) \leq B$, $\forall j$, under the event of Lemma \ref{LEMMA:VARIANCE ESTIMATOR} and the definition of $\bar{\sigma}_{t,l}^{-2}$. Following the inequality that $\det \Ab \leq \|\Ab\|_2^n$, where $\Ab$ is any $n \times n$ matrix, we have $\operatorname{det}\left(\tilde{\bSigma}_{T,l}\right) \leq\left(\lambda+T^2\right)^d$. Furthermore, we have
\begin{align*}
\left(\lambda+T^2 \right)^d \geq 2^{J_{1,l}} \cdot \operatorname{det}\left(\boldsymbol{\Sigma}_0\right)=2^{J_{1,l}} \cdot \lambda^d,
\end{align*}
 From the above inequality, we conclude that
\begin{align*}
J_{1,l} \leq 2 d \log \left(1+\frac{T^2}{\lambda}\right) \leq 4d \log \left(1+\frac{T}{\lambda}\right).
\end{align*}
Note that this bound does not depend on $l$, so we take a summation over all $l\in[L]$ and get the bound of $J_1$ as
\begin{align*}
J_{1} \leq 4dL \log \left(1+\frac{T}{\lambda}\right).
\end{align*}
To bound $J_2$, note that $t_0=1$ and thus $2^{J_2} \leq T$, which immediately gives $J_2 \leq \log _2 T \leq 2 \log T$. Altogether we conclude that
\begin{align*}
J=J_1+J_2 \leq 4 dL \log \left(1+\frac{T}{\lambda}\right)+2 \log T.
\end{align*}
Thus, we complete the proof of Lemma \ref{lemma:J}.
\end{proof}
\begin{proof}[Proof of Lemma \ref{lemma:E3}]
We first consider the first term on the left-hand side. After rearranging the summation, we can find that the following equation holds,
\begin{align*}
& \sum_{m=1}^M\sum_{h=1}^{H_m} \left[V_{j_m}\left(s_{m, h}\right)-V_{j_m}\left(s_{m, h+1}\right)\right] \\
&= \sum_{m=1}^M V_{j_m}\left(s_{m, 1}\right)-V_{j_m}\left(s_{m, H_m+1}\right) \\
&= \sum_{m=1}^{M-1}\left(V_{j_{m+1}}\left(s_{m+1,1}\right)-V_{j_m}\left(s_{m, H_m+1}\right)\right)+\sum_{m=1}^{M-1}\left(V_{j_m}\left(s_{m, 1}\right)-V_{j_{m+1}}\left(s_{m+1,1}\right)\right) \\
& \qquad +V_{j_M}\left(s_{M, 1}\right)-V_{j_M}\left(s_{M, H_M+1}\right).
\end{align*}
Using the telescope argument to the second term in the above equation, we have the following inequality,
\begin{align}
\notag& \sum_{m=1}^M\sum_{h=1}^{H_m}\left[ V_{j_m}\left(s_{m, h}\right)-V_{j_m}\left(s_{m, h+1}\right)\right] \\
\notag& =\sum_{m=1}^{M-1}\left(V_{j_{m+1}}\left(s_{m+1,1}\right)-V_{j_m}\left(s_{m, H_m+1}\right)\right)+V_{j_1}\left(s_{1,1}\right)-V_{j_M}\left(s_{M, 1}\right) \\
\notag& \quad \quad \quad+V_{j_M}\left(s_{M, 1}\right)-V_{j_M}\left(s_{M, H_M+1}\right) \\
\notag& =\sum_{m=1}^{M-1}\left(V_{j_{m+1}}\left(s_{m+1,1}\right)-V_{j_m}\left(s_{m, H_m+1}\right)\right)+V_{j_1}\left(s_{1,1}\right)-V_{j_M}\left(s_{M, H_M+1}\right) \\
\notag& \leq \sum_{m=1}^{M-1}\left(V_{j_{m+1}}\left(s_{m+1,1}\right)-V_{j_m}\left(s_{m, H_m+1}\right)\right)+V_{j_1}\left(s_{1,1}\right),
\end{align}
where the last inequality holds because $V_j(\cdot)$ is non-negative for all $j \in \NN$ from the design of DEVI algorithm.

We now consider the term $V_{j_{m+1}}\left(s_{m+1,1}\right)-V_{j_m}\left(s_{m, H_m+1}\right)$. Note that by the interval decomposition, interval $m$ ends if and only if the updating criterion of the DEVI algorithm is met or the goal state is reached. In addition, if interval $m$ ends because the goal state is reached, then we have
\begin{align*}
V_{j_{m+1}}\left(s_{m+1,1}\right)-V_{j_m}\left(s_{m, H_m+1}\right)=V_{j_{m+1}}\left(s_{\text {init }}\right)-V_{j_m}(g)=V_{j_{m+1}}\left(s_{\text {init }}\right).
\end{align*}
If it ends because the updating criterion of the DEVI algorithm is triggered, then the value function is updated by DEVI and $j_m \neq j_{m+1}$. In such case, we simply apply the trivial upper bound $V_{j_{m+1}}\left(s_{m+1,1}\right)-V_{j_m}\left(s_{m, H_m+1}\right) \leq \max _j\left\|V_j\right\|_{\infty}$. According to Lemma \ref{lemma:J}, this happens at most $J \leq 4 dL \log \left(1+{T}/{\lambda}\right)+2 \log T$ times. Therefore, we can further bound the term $\sum_{m=1}^{M-1}\left(V_{j_{m+1}}\left(s_{m+1,1}\right)-V_{j_m}\left(s_{m, H_m+1}\right)\right)$ as
\begin{align*}
& \sum_{m=1}^M\sum_{h=1}^{H_m} \left(V_{j_m}\left(s_{m, h}\right)-V_{j_m}\left(s_{m, h+1}\right)\right) \\
& \leq \sum_{m=1}^{M-1} V_{j_{m+1}}\left(s_{\text {init }}\right) \cdot \ind\{m+1 \in \mathcal{M}(M)\}+V_{j_1}\left(s_{1,1}\right)+\left[4 dL \log \left(1+\frac{T}{\lambda}\right)+2 \log T\right] \cdot \max _j\left\|V_j\right\|_{\infty} \\
& \leq \sum_{m \in \mathcal{M}(M)} V_{j_m}\left(s_{\text {init }}\right)+V_0\left(s_{\text {init }}\right)+4 d B_{*}L \log \left(1+\frac{T}{\lambda}\right)+2 B_{*} \log T \\
& \leq \sum_{m \in \mathcal{M}(M)} V_{j_m}\left(s_{\text {init }}\right)+1+4 d B_{*} \log \left(1+\frac{T }{\lambda}\right)+2 B_{*} \log T,
\end{align*}
where the second inequality holds due to $\left\|V_j\right\|_{\infty} \leq B_{*}$, under the event of Lemma \ref{LEMMA:VARIANCE ESTIMATOR}, and the last inequality holds because of our initialization $\left\|V_0\right\|_{\infty} \leq 1$. Thus, we complete the proof of Lemma \ref{lemma:E3}.
\end{proof}
\begin{proof}[Proof of Lemma \ref{lemma:R}]
    For any  level $l$, applying Lemma \ref{lemma:sum} with $\xb_{t} = \bphi_{V_j^{2^l}}(s_t,a_t)/B^{2^l}$,  $\sigma_{t}^2 = \Big( \big[\bar{\VV}_{t,l}V_{j+1}^{2^l}](s_t,a_t)/B^{2^{l+1}} + E_{t,l}\Big)$, $\alpha_t^{\prime 2} =   \alpha_t^2$ and
    $\gamma^{\prime 2} = \gamma^2$, where $\alpha_t$,$\gamma$ are parameters used to construct the weights $\bar{\sigma}_{t,l}$ in Algorithm \ref{algorithm3}, $\beta_{t} = \hat{\beta}_T$ is defined in \eqref{eq:beta}, we have the following results,
    \begin{align*}
        R_l &= \sum_{m \in \cM_0(M)}\sum_{h=1}^{H_m}\min \Big\{1,\hat{\beta}_{T}\big\|\bphi_{V_{j_m}^{2^l}}(s_{m,h},a_{m,h})/B^{2^{l}}\big\|_{\tilde{\bSigma}_{t,l}^{-1}}\Big\}\\
        &\leq 2d\iota + 2\hat{\beta}_T \gamma^2d \iota + 2\sqrt{d\iota}\hat{\beta}_T\sqrt{ \sum_{m \in \cM_0(M)}\sum_{h=1}^{H_m}\alpha_{t(m,h)}^2 + \sum_{m \in \cM_0(M)}\sum_{h=1}^{H_m}\sigma_{t,l}^2}.
    \end{align*}
    Thus, we complete the proof of Lemma \ref{lemma:R}.
\end{proof}
\end{lemma}
\begin{proof}[Proof of Lemma \ref{Lemma:sigma}]
Use the definition of variance $\sigma_{t,l}$, we have the following inequality,
\begin{align*}
    \sum_{m \in \cM_0}\sum_{h=1}^{H_m}\sigma_{t,l}^2
    &= \sum_{m \in \cM_0}\sum_{h=1}^{H_m}E_{t,l} + \sum_{m \in \cM_0}\sum_{h=1}^{H_m}[\bar{\VV}_t V_{j_m}^{2^l}](s_{m,h},a_{m,h})/B^{2^{l+1}}\\
    &= 2\sum_{m \in \cM_0}\sum_{h=1}^{H_m}E_{t,l} + \sum_{m \in \cM_0}\sum_{h=1}^{H_m}\VV V_{j_m}^{2^{l}}(s_{m,h},a_{m,h})/B^{2^{l+1}}\\
    & \qquad + B^{-2^{l+1}}\sum_{m \in \cM_0}\sum_{h=1}^{H_m}\big([\bar{\VV}_t V_{j_m}^{2^l}](s_{m,h},a_{m,h}) - B^{2^{l+1}}E_{t,l} - \VV V_{j_m}^{2^l}(s_{m,h},a_{m,h})\big)\\
    &\leq 2\sum_{m \in \cM_0}\sum_{h=1}^{H_m}E_{t,l} + \sum_{m \in \cM_0}\sum_{h=1}^{H_m}\VV V_{j_m}^{2^l}(s_{m,h},a_{m,h})/B^{2^{l+1}},
\end{align*}
where the last inequality holds since $\sum_{m \in \cM_0}\sum_{h=1}^{H_m}\big([\bar{\VV}_t V_{j_m}^{2^l}](s_{m,h},a_{m,h}) - B^{2^{l+1}}E_{t,l} - \VV V_{j_m}^{2^l}(s_{m,h},a_{m,h})\big) \leq 0$ under the event of Lemma \ref{LEMMA:VARIANCE ESTIMATOR}.
\end{proof}
\begin{proof}[Proof of Lemma \ref{Lemma:S}]
According the definition of $S_l$ and the convexity of square function $x^{2}$, we can rearrange the summation in $E_l$ into three terms and obtain the following inequality,
\begin{align*}
    S_l &= \sum_{m \in \cM_0}\sum_{h=1}^{H_m}\bigg[\PP V_{j_m}^{2^{l+1}}(s_{m,h},a_{m,h})/B^{2^{l+1}}-\big(\PP V_{j_m}^{2^l}(s_{m,h},a_{m,h})/B^{2^l}\big)^2\bigg]\\
    & \leq \sum_{m \in \cM_0}\sum_{h=1}^{H_m}\bigg[\PP V_{j_m}^{2^{l+1}}(s_{m,h},a_{m,h})/B^{2^{l+1}}-\big(\PP V_{j_m}(s_{m,h},a_{m,h})/B\big)^{2^{l+1}}\bigg]\\
    & = \sum_{m \in \cM_0}\sum_{h=1}^{H_m}\bigg[\PP V_{j_m}^{2^{l+1}}(s_{m,h},a_{m,h})/B^{2^{l+1}}- V_{j_m}^{2^{l+1}}(s_{m,h+1})/B^{2^{l+1}}\bigg]\\
    & \qquad + \sum_{m \in \cM_0}\sum_{h=1}^{H_m}\bigg[ V_{j_m}^{2^{l+1}}(s_{m,h})/B^{2^{l+1}}- \big(\PP V_{j_m}(s_{m,h},a_{m,h})/B\big)^{2^{l+1}}\bigg]\\
    & \qquad + \sum_{m \in \cM_0}\sum_{h=1}^{H_m}\bigg[ V_{j_m}^{2^{l+1}}(s_{m,h+1})/B^{2^{l+1}}- V_{j_m}^{2^{l+1}}(s_{m,h})/B^{2^{l+1}}\bigg].
    \end{align*}
    Recalling the definition of $A_l$, the first term is simply equal to $A_{l+1}$. For the third term, we can use the telescope argument and get 
    \begin{align*}
        \sum_{m \in \cM_0}\sum_{h=1}^{H_m}\left[ V_{j_m}^{2^{l+1}}(s_{m,h+1})/B^{2^{l+1}}- V_{j_m}^{2^{l+1}}(s_{m,h})/B^{2^{l+1}}\right] \leq |\cM_0|.
    \end{align*}
    For the second term, we can inductively degrade the exponents of the value function. Combining the bound of these three terms, we have the following inequality 
    \begin{align*}
    S_l & = A_{l+1} + \sum_{m \in \cM_0}\sum_{h=1}^{H_m}\bigg[ V_{j_m}^{2^{l+1}}(s_{m,h})/B^{2^{l+1}}- \big(\PP V_{j_m}(s_{m,h+1})/B\big)^{2^{l+1}}\bigg]\\
    & \qquad + \sum_{m \in \cM_0}\bigg[ V_{j_m}^{2^{l+1}}(s_{m,H_m+1})/B^{2^{l+1}}- V_{j_m}^{2^{l+1}}(s_{m,1})/B^{2^{l+1}}\bigg]\\
    & \leq A_{l+1} + |\cM_0| + \sum_{m \in \cM_0}\sum_{h=1}^{H_m}\bigg[ V_{j_m}^{2^{l}}(s_{m,h})/B^{2^{l}}+ \big(\PP V_{j_m}(s_{m,h+1})/B\big)^{2^{l}}\bigg]\\
    & \qquad \cdot \bigg[ V_{j_m}^{2^{l}}(s_{m,h})/B^{2^{l}}- \big(\PP V_{j_m}(s_{m,h+1})/B\big)^{2^{l}}\bigg]\\
    &\leq A_{l+1} + |\cM_0| + 2\sum_{m \in \cM_0}\sum_{h=1}^{H_m}\bigg[ V_{j_m}^{2^l}(s_{m,h})/B^{2^l}- \big(\PP V_{j_m}(s_{m,h},a_{m,h})/B\big)^{2^l}\bigg]\\
    & \leq \ldots\\
    & \leq A_{l+1} + |\cM_0| + 2^{l+1}\sum_{m \in \cM_0}\sum_{h=1}^{H_m}\bigg[ V_{j_m}(s_{m,h})/B- \PP V_{j_m}(s_{m,h},a_{m,h})/B\bigg],
    \end{align*}
    where we use the fact $V_j(s) \leq B, \forall j \geq 1$ under the event of Lemma \ref{LEMMA:VARIANCE ESTIMATOR}.
    Note that the sum is just similar to the term $I_1$ in Lemma \ref{LEMMA:REGRET DECOMPOSITION} when we first decompose the regret. So we have the following inequality,
\begin{align*}
    S_l & \leq A_{l+1} + |\cM_0| + 2^{l+1}/B \underbrace{\sum_{m \in \cM_0}\sum_{h=1}^{H_m}c(s_{m,h},a_{m,h})}_{\leq C_M}\\
    & \qquad + 2^{l+1}\underbrace{\sum_{m \in \cM_0}\sum_{h=1}^{H_m}\bigg[ V_{j_m}(s_{m,h})/B- \PP V_{j_m}(s_{m,h},a_{m,h})/B - c(s_{m,h},a_{m,h})/B\bigg] }_{-I_1/B}\\
    & \leq A_{l+1} + |\cM_0| + \frac{2^{l+1}}{B} C_M + \frac{2^{l+1}}{B} (B_1+B_2)\\
    & \leq A_{l+1} + |\cM_0| + \frac{2^{l+1}}{B} C_M + \frac{2^{l+1}}{B} (4BR_0+B_2),
\end{align*}
where we use Lemma \ref{LEMMA:E1} and the observation $B_1 \leq 4B R_0$.
\end{proof}
\begin{proof}[Proof of Lemma \ref{lemma:Bernstein}]
We follow the proof of Lemma 25 in \citet{zhang2021improved}. We use Lemma \ref{lemma:zhang} for each fixed level $l$. Let $x_{m, h}=\Big[\big[\mathbb{P} V_{j_m}^{2^l}\big]\left(s_{m,h}, a_{m,h}\right)-V_{j_m}^{2^l}\left(s_{m,h+1}\right)\Big]/B^{2^l}$, then we have $\mathbb{E}\left[x_{m, h} \mid \mathcal{G}_{m, h}\right]=0$ and $\mathbb{E}\left[x_{m, h}^2\Big| \mathcal{G}_{m, h}\right]=\left[\mathbb{V} V_{j_m}^{2^l}\right]\left(s_{m,h}, a_{m,h}\right)/B^{2^{l+1}}$. Therefore, for each level $l \in [L]$, with probability at least $1-\delta$, we have
\begin{align*}
A_l=\sum_{m \in \cM_0(M)}\sum_{h=1}^{H_m} x_{m, h} &\leq \sqrt{2 \zeta \sum_{m \in \cM_0(M)}\sum_{h=1}^{H_m} \mathbb{V} V_{j_m}^{2^l}\left(s_{m,h}, a_{m,h}\right)/B^{2^{l+1}}}+\zeta\\\notag
& = \sqrt{2\zeta S_l} + \zeta, \notag
\end{align*}
where $\zeta = 4\log{(2\log (T\log (1/\delta))+1)/\delta)}$.
Taking a union bound over $l \in [L]$, we complete the proof of Lemma \ref{lemma:Bernstein}.
\end{proof}

\section{Auxiliary Lemmas}

\begin{theorem}
[Theorem 4.3 in \citealt{zhou2022computationally}]\label{thm:Bernstein}
    Let $\big\{\mathcal{G}_k\big\}_{k=1}^{\infty}$ be a filtration, and $\left\{\mathbf{x}_k, \eta_k\right\}_{k \geq 1}$ be a stochastic process such that $\mathbf{x}_k \in \mathbb{R}^d$ is $\mathcal{G}_k$-measurable and $\eta_k \in \mathbb{R}$ is $\mathcal{G}_{k+1}$-measurable. Let $L, \sigma, \lambda, \epsilon>0, \boldsymbol{\mu}^* \in \mathbb{R}^d$. For $k \geq 1$, let $y_k=\left\langle\boldsymbol{\mu}^*, \mathbf{x}_k\right\rangle+\eta_k$ and suppose that $\eta_k, \mathbf{x}_k$ also satisfy
\begin{align*}
\left.\mathbb{E}\left[\eta_k\right| \mathcal{G}_k\right]=0, \mathbb{E}\left[\eta_k^2 \mid \mathcal{G}_k\right] \leq \sigma^2,\left|\eta_k\right| \leq R,\left\|\mathbf{x}_k\right\|_2 \leq L.
\end{align*}
For $k \geq 1$, let $\mathbf{Z}_k=\lambda \mathbf{I}+\sum_{i=1}^k \mathbf{x}_i \mathbf{x}_i^{\top}, \mathbf{b}_k=\sum_{i=1}^k y_i \mathbf{x}_i, \boldsymbol{\mu}_k=\mathbf{Z}_k^{-1} \mathbf{b}_k$, and
\begin{align*}
\beta_k=& 12 \sqrt{\sigma^2 d \log \left(1+k L^2 /(d \lambda)\right) \log \left(32(\log (R / \epsilon)+1) k^2 / \delta\right)} \\
&+24 \log \left(32(\log (R / \epsilon)+1) k^2 / \delta\right) \max _{1 \leq i \leq k}\left\{\left|\eta_i\right| \min \left\{1,\left\|\mathbf{x}_i\right\|_{\mathbf{z}_{i-1}^{-1}}\right\}\right\}+6 \log \left(32(\log (R / \epsilon)+1) k^2 / \delta\right) \epsilon.
\end{align*}
Then, for any $0<\delta<1$, we have with probability at least $1-\delta$ that,
\begin{align*}
\forall k \geq 1,\left\|\sum_{i=1}^k \mathbf{x}_i \eta_i\right\|_{\mathbf{Z}_k^{-1}} \leq \beta_k,\left\|\boldsymbol{\mu}_k-\boldsymbol{\mu}^*\right\|_{\mathbf{Z}_k} \leq \beta_k+\sqrt{\lambda}\left\|\boldsymbol{\mu}^*\right\|_2.
\end{align*}
\end{theorem}
\begin{lemma}[Lemma B.1 in \citealt{zhou2022computationally}]\label{lemma:sum}
    Let $\left\{\sigma_t, \beta_t\right\}_{t \geq 1}$ be a sequence of non-negative numbers, $\alpha_t^{\prime} > 0$ decreasing, $\gamma^{\prime}>0,\left\{\mathbf{x}_t\right\}_{t \geq 1} \subset \mathbb{R}^d$ and $\left\|\mathbf{x}_t\right\|_2 \leq L$. Let $\left\{\mathbf{Z}_t\right\}_{t \geq 1}$ and $\left\{\bar{\sigma}_t\right\}_{t \geq 1}$ be recursively defined as follows: $\mathbf{Z}_1=\lambda \mathbf{I}$
\begin{align*}
\forall t \geq 1, \bar{\sigma}_t=\max \left\{\sigma_t, \alpha_t^{\prime}, \gamma^{\prime}\left\|\mathbf{x}_t\right\|_{\mathbf{Z}_t^{-1}}^{1 / 2}\right\}, \mathbf{Z}_{t+1}=\mathbf{Z}_t+\mathbf{x}_t \mathbf{x}_t^{\top} / \bar{\sigma}_t^2 .
\end{align*}
Let $\iota=\log \left(1+T L^2 /\left(d \lambda \alpha_T^{\prime 2}\right)\right)$. Then we have
\begin{align*}
\sum_{t=1}^T \min \left\{1, \beta_t\left\|\mathbf{x}_t\right\|_{\mathbf{Z}_t^{-1}}\right\} \leq 2 d \iota+2 \max _{t} \beta_t \gamma^{\prime 2} d \iota+2 \sqrt{d \iota} \sqrt{\sum_{t} \beta_t^2\left(\sigma_t^2+\alpha_t^{\prime 2}\right)}.
\end{align*}
\end{lemma}
\begin{lemma}
    [Lemma 11 in \citealt{zhang2021model}]\label{lemma:zhang}Let $M>0$ be a constant. Let $\left\{x_i\right\}_{i=1}^n$ be a stochastic process, $\mathcal{G}_i=\sigma\left(x_1, \ldots, x_i\right)$ be the $\sigma$-algebra of $x_1, \ldots, x_i$. Suppose $\mathbb{E}\left[x_i \mid \mathcal{G}_{i-1}\right]=0,\left|x_i\right| \leq M$ and $\left.\mathbb{E}\left|x_i^2\right| \mathcal{G}_{i-1}\right]<\infty$ almost surely. Then, for any $\delta, \epsilon>0$, we have
\begin{align*}
&\mathbb{P}\left(\left|\sum_{i=1}^n x_i\right| \leq 2 \sqrt{2 \log (1 / \delta) \sum_{i=1}^n \mathbb{E}\left[x_i^2 \mid \mathcal{G}_{i-1}\right]}+2 \sqrt{\log (1 / \delta)} \epsilon+2 M \log (1 / \delta)\right) \\
&\quad>1-2\left(\log \left(M^2 n / \epsilon^2\right)+1\right) \delta.
\end{align*}
\end{lemma}
\begin{lemma}\label{lemma:zhangnew}
    Let $\lambda_1, \lambda_2, \lambda_3, \lambda_4>0$ and $\kappa \geq \max \left\{\log _2 (\lambda_1/\lambda_3), 1\right\}$. Let $a_1, \ldots, a_n$ be non-negative real numbers such that $a_i \leq \min \left\{\lambda_1, \lambda_2 \sqrt{a_i+a_{i+1}+2^{i+1} \lambda_3}+\lambda_4\right\}$ for any $1 \leq i \leq \kappa$. Let $a_{\kappa+1}=\lambda_1$. Then we have $a_1 \leq 22 \lambda_2^2+6 \lambda_4+4 \lambda_2 \sqrt{2 \lambda_3}$.
\end{lemma}
\begin{proof}[Proof of Lemma \ref{lemma:zhangnew}]
    Define a new sequence $b_l = a_l/\lambda_3$. Using Lemma \ref{lemma:zhang2} with sequence $\{b_l\}$ and parameters $\lambda_1/\lambda_3, \lambda_2/\lambda_3,1,\lambda_4/\lambda_3$, we finish the proof of Lemma \ref{lemma:zhangnew}.
\end{proof}
\begin{lemma}
    [Lemma 12 in \citealt{zhang2021improved}]\label{lemma:zhang2} Let $\lambda_1, \lambda_2, \lambda_4>0, \lambda_3 \geq 1$ and $\kappa=\max \left\{\log _2 \lambda_1, 1\right\}$. Let $a_1, \ldots, a_n$ be non-negative real numbers such that $a_i \leq \min \left\{\lambda_1, \lambda_2 \sqrt{a_i+a_{i+1}+2^{i+1} \lambda_3}+\lambda_4\right\}$ for any $1 \leq i \leq \kappa$. Let $a_{\kappa+1}=\lambda_1$. Then we have $a_1 \leq 22 \lambda_2^2+6 \lambda_4+4 \lambda_2 \sqrt{2 \lambda_3}$.
\end{lemma}
\begin{lemma}
    [Lemma 12 in \citealt{abbasi2011improved}]\label{lemma:det} Suppose $\mathbf{A}, \mathbf{B} \in \mathbb{R}^{d \times d}$ are two positive definite matrices satisfying $\mathbf{A} \succeq \mathbf{B}$, then for any $\mathbf{x} \in \mathbb{R}^d,\|\mathbf{x}\|_{\mathbf{A}} \leq\|\mathbf{x}\|_{\mathbf{B}} \cdot \sqrt{\operatorname{det}(\mathbf{A}) / \operatorname{det}(\mathbf{B})}$.
\end{lemma}

\end{document}